\pdfoutput=1
\documentclass{article}
\usepackage[T1]{fontenc}
\usepackage{iclr2027_conference,times}
\iclrfinalcopy

\usepackage{amsmath}
\usepackage{amsfonts}
\usepackage{amssymb}
\usepackage{amsthm}
\usepackage{booktabs}
\usepackage{graphicx}
\usepackage{float}
\usepackage{placeins}
\usepackage{array}
\usepackage{longtable}
\usepackage{caption}
\usepackage{hyperref}
\usepackage{url}
\usepackage{xcolor}
\ifdefined\pdfgentounicode
  \IfFileExists{glyphtounicode.tex}{\input{glyphtounicode}}{}
\fi

\newtheorem{theorem}{Theorem}
\newtheorem{lemma}{Lemma}
\newtheorem{proposition}{Proposition}

\DeclareMathOperator{\Tr}{tr}
\DeclareMathOperator{\diag}{diag}
\newcommand{\R}{\mathbb{R}}
\newcommand{\norm}[1]{\left\lVert #1 \right\rVert}
\newcommand{\cen}{\mathrm{cen}}

\title{
Understanding Head Geometry and Dynamics in Federated Regression \\
 through a Natural Solution Selection Rule:
 \\ An Unconstrained Feature Model Analysis
}
\author{Chuang Ma$^{1,2}$ \quad Tomoyuki Obuchi$^{1,3}$\\
$^{1}$Kyoto University \quad $^{2}$NII LLMC \quad $^{3}$RIKEN AIP\\
\texttt{\{ma.chuang.52h@st, obuchi@i\}.kyoto-u.ac.jp}}
\begin{document}
\maketitle
\lhead{Preprint}
\begin{abstract}
In federated averaging, local objectives can admit multiple optimal heads, making the aggregate depend on which heads clients return.
We study this ambiguity in federated multivariate regression with private backbones and a shared linear head, using an unconstrained feature model (UFM) that treats training-sample features as free variables.
We introduce a natural selection rule: each client returns the optimal head closest to the broadcast head.
We show that global minimization with a vanishing proximal penalty on the head realizes this rule.
When the clients' optimal Gram matrices and the initial shared Gram matrix are positive definite, the shared Gram matrix follows a closed recursion and converges to the unique Bures--Wasserstein barycenter of the clients' optimal Gram matrices.
Even with this alignment, the limit generally differs from the centralized optimal Gram matrix.
We decompose this gap into three positive-semidefinite terms arising from differences in client target means, covariance heterogeneity, and averaging the aligned heads.
A correction based on a one-time exchange of target means and covariances recovers the centralized optimal Gram matrix in one round under exact local optimization and the same selection rule.
We verify these results numerically in the UFM and test its predictions on five tabular and five image regression datasets using deep networks with feature regularization and long local training.
In these experiments, ordinary training approaches the predicted barycenter, while a weak proximal penalty improves endpoint agreement and yields trajectories that closely follow the predicted Gram dynamics.
The correction moves the final Gram matrices close to the centralized UFM prediction.
\end{abstract}
\section{Introduction}
\label{sec:intro}
In federated averaging, clients update the shared model
parameters using their local data at each communication
round and send the updated parameters to the server.
The server then averages these parameters and broadcasts
the result~\citep{pmlr-v54-mcmahan17a}.
Effective aggregation relies on alignment between client
models in parameter space, but local objectives alone
do not guarantee such alignment.
Prior work has characterized the limiting behavior of
federated averaging in strongly convex settings, where
local minimizers are unique~\citep{pmlr-v130-charles21a,pmlr-v258-mangold25a}.
In overparameterized linear regression with fixed features,
gradient descent selects the local minimizer closest to
the broadcast model, and repeated averaging can recover
the centralized solution~\citep{zhu2026effectivenessdistributedgradientdescent}.

In contrast, local training of deep neural networks need
not return the minimizer closest to the broadcast model.
Different parameter settings can represent the same
function~\citep{Wang2020Federated}, so clients may return
mutually incompatible solutions.
For example, permuting hidden units can cause the same
unit index to correspond to different features across clients.
Averaging such solutions can produce an arbitrarily poor
model~\citep{pmlr-v54-mcmahan17a}. Indeed, nonzero local updates can largely cancel upon averaging~\citep{jhunjhunwala2023fedexp}, and averaging-based optimization methods can stagnate at nonoptimal
points~\citep{NEURIPS2020_4ebd440d,NEURIPS2019_fea16e78}. To address mismatches in parameter representations, prior work has proposed matching methods across clients before averaging~\citep{Wang2020Federated,NEURIPS2020_fb269786}. These studies, however, did not characterize the limiting behavior of repeated averaging or its relation to centralized training.

\begin{figure}[t]
 \centering
 \includegraphics[width=\linewidth]{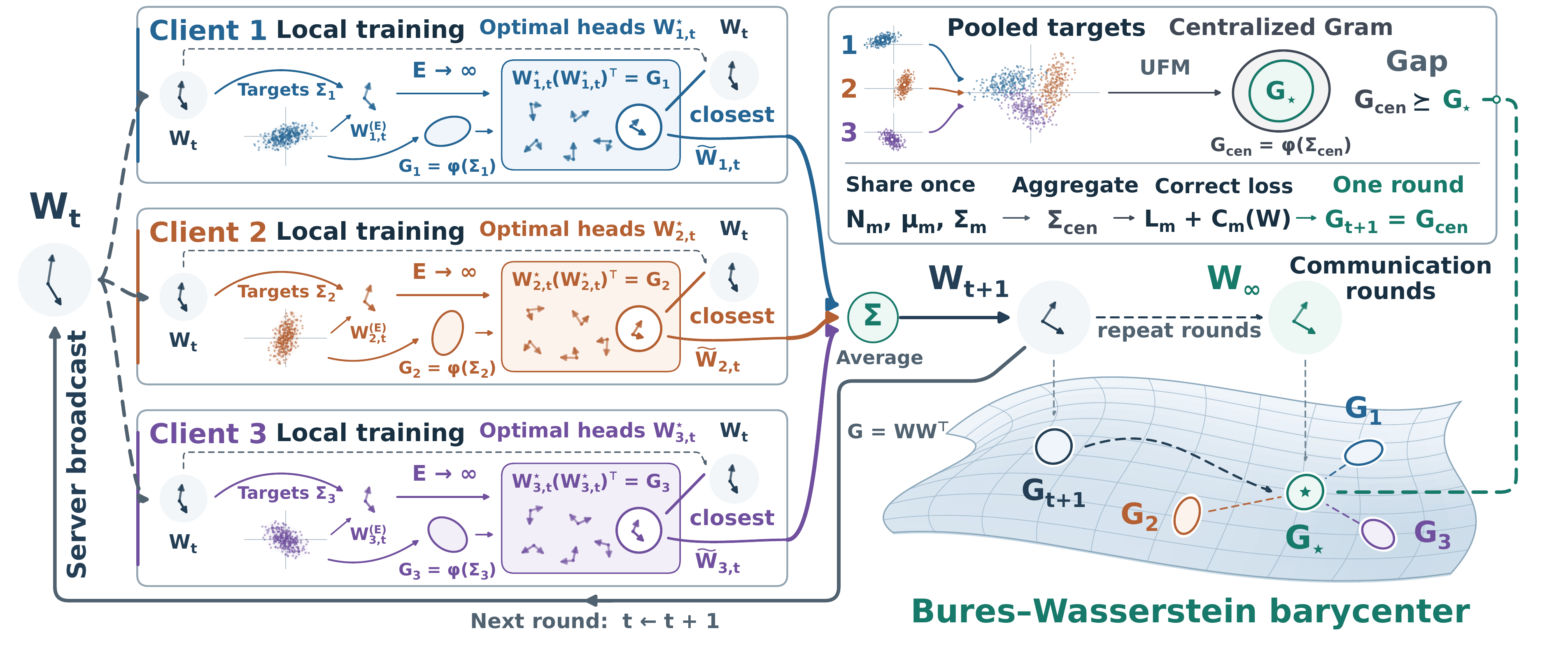}
 \caption{\textbf{Federated averaging of a shared head under the selection rule.}
 In the UFM, local optimality fixes the Gram matrix of a client's head but not the head itself.
 Our selection rule chooses the optimal head closest to the broadcast head.
 Averaging these heads over communication rounds drives the Gram matrix of the shared head to the BW barycenter of the clients' optimal Gram matrices.
 Inset: the gap to centralized training and its one-round correction from shared target moments.}
 \label{fig:intro-overview}
\end{figure}
To study the limiting behavior of federated averaging when local minimizers are nonunique, we focus on the unconstrained feature model (UFM)~\citep{Mixon2022,doi:10.1073/pnas.2103091118,NEURIPS2021_f92586a2,pmlr-v145-e22b,LU2022224}. This model treats the backbone outputs on the training samples as free optimization variables. Originally introduced to explain neural collapse~\citep{doi:10.1073/pnas.2015509117} in classification, it was later extended to regression~\citep{NEURIPS2024_e4748b6b,NEURIPS2025_93da65de}. For regularized multivariate regression, \citet{NEURIPS2024_e4748b6b} derived an explicit expression for the Gram matrix of an optimal linear head in terms of the target covariance. The head itself, however, is determined only up to an orthogonal transformation.

Building on this characterization, we introduce and analyze a UFM for federated multivariate regression with private backbones and a shared linear head~\citep{liang2020thinklocallyactglobally,10.1145/3545008.3545073}. Each client's features are optimized independently, and only the head is averaged (Figure~\ref{fig:intro-overview}). The orthogonal freedom, harmless for an individual model, now complicates aggregation because the server averages the heads rather than their Gram matrices. To resolve the ambiguity in the client returns, we select, for each client, the optimal head closest to the broadcast head in Frobenius norm. This uses the broadcast as a common reference while preserving optimality for each client's original objective. We further show that this selection arises as a vanishing-regularization limit. Specifically, we add a quadratic proximal penalty on the head, of the form used in FedProx~\citep{MLSYS2020_1f5fe839}, and show that the heads of global minimizers of the penalized local objective converge to the selected head as the penalty weight tends to zero.

Under this selection rule, we identify the limiting Gram matrix of the shared head as the Bures--Wasserstein (BW) barycenter of the clients' optimal Gram matrices. When these matrices are positive definite, we prove
convergence to this unique barycenter from any positive-definite initial Gram matrix. The key connection is that optimal head alignment in Frobenius norm corresponds to BW distance between Gram matrices~\citep{BHATIA2019165}: averaging the selected heads induces the barycenter iteration of \citet{ALVAREZESTEBAN2016744}.

Even with exact local optimization and aligned heads, this barycenter generally differs from the optimal Gram matrix of the centralized UFM trained on the pooled data, i.e., the combined data from all clients. We express the centralized Gram matrix as the barycenter plus three positive-semidefinite terms, corresponding to differences in client target means, the nonlinear dependence of optimal Gram matrices on target covariances, and averaging heads rather than their Gram matrices. We further construct a correction to the local objectives that recovers the centralized Gram matrix in a single communication round under the same selection rule. The correction requires a one-time exchange of the clients' target means and covariances (Figure~\ref{fig:intro-overview}, inset).

We first verify these results numerically in the UFM.
To test whether they extend to learned features, we train deep networks with the feature and head penalties of the UFM on five tabular and five image regression datasets.
With a weak proximal regularizer, the Gram matrix of the shared head closely follows the predicted communication dynamics and ends near the BW barycenter, and the correction moves this endpoint close to the centralized Gram matrix.
Without the regularizer, the endpoint also lies near the barycenter but less precisely, and an intermediate proximal weight agrees best.
Aligning the heads to the broadcast also implements the selection rule, whereas averaging unaligned heads raises the training objective.
Removing the between-client mean term leaves a gap consistent with the decomposition.
The agreement persists for other architectures, larger training sets and full-model averaging, and weakens with less local training, as our analysis of finite local steps in the UFM also shows (Appendix~\ref{sub:supplementary-analysis}).
On the image datasets, proximal training keeps test error comparable, and the correction lowers it slightly.

\section{Related work}
\label{sec:related}
\paragraph{Unconstrained feature models.}
The UFM and the layer-peeled model explain neural collapse through free-feature optimization~\citep{Mixon2022,doi:10.1073/pnas.2103091118}.
Prior work characterizes their global minimizers, landscapes and dynamics under cross-entropy and squared losses~\citep{NEURIPS2021_f92586a2,pmlr-v162-zhou22c,han2022neural}.
Extensions address class imbalance~\citep{doi:10.1073/pnas.2103091118,NEURIPS2022_ae54ce31,JMLR:v25:23-1215}, deeper models~\citep{pmlr-v162-tirer22a,NEURIPS2023_a60c43ba,NEURIPS2024_f9c2ab8d} and constrained features~\citep{pmlr-v202-tirer23a}.
Related analyses examine neural collapse in shallow ReLU~\citep{JMLR:v27:24-1429} and mean-field networks~\citep{pmlr-v267-wu25u}.
The UFM has also been extended to multi-label classification~\citep{pmlr-v235-li24ai}, supervised contrastive learning~\citep{10619192}, ordinal regression~\citep{NEURIPS2025_93da65de} and multivariate regression~\citep{NEURIPS2024_e4748b6b,andriopoulos2025neuralmultivariateregressionqualitative}.
In federated classification, related analyses study feature collapse and degradation under aggregation~\citep{10.1007/978-981-99-8132-8_34,shi2023towards,zhu2025coinunveilingdownsidesmodel}, and other methods prescribe the classifier geometry~\citep{Li_2023_ICCV,ICLR2024_db174d37}.
We build on the characterization of multivariate regression~\citep{NEURIPS2024_e4748b6b} to analyze repeated averaging of selected optimal heads across clients.
\paragraph{Federated optimization and selection.}
Prior analyses identify objective inconsistency and biased fixed points~\citep{NEURIPS2020_564127c0,NEURIPS2020_4ebd440d}.
Surrogate objectives describe local updates for quadratic losses~\citep{pmlr-v130-charles21a}, and bias expansions quantify the limiting deviation from centralized training in strongly convex settings~\citep{pmlr-v258-mangold25a}.
Multiple local updates also enable shared representation learning in multitask linear regression~\citep{NEURIPS2022_449590df}.
With fixed features, local gradient descent in overparameterized linear regression selects the minimizer closest to its broadcast initialization~\citep{zhu2026effectivenessdistributedgradientdescent}.
Matching methods explicitly align models before averaging~\citep{Wang2020Federated,NEURIPS2020_fb269786}, while FedProx penalizes deviations from the broadcast model~\citep{MLSYS2020_1f5fe839}.
In our UFM, the heads of global minimizers of the penalized objective converge to the selected head as the proximal weight vanishes.
For classifier calibration, CCVR uses uploaded feature moments~\citep{NEURIPS2021_2f2b2656}; our correction instead uses target means and covariances to recover the centralized Gram matrix.
\paragraph{Optimal transport and matrix barycenters.}
For positive-definite matrices, the BW distance is the Procrustes distance between their factors~\citep{BHATIA2019165}.
\citet{berardini2026distributionalignmentoneshotfederated} use BW barycenters to align frozen-encoder feature distributions in one-shot federated learning.
Under our selection rule, repeated head averaging instead induces the convergent BW barycenter iteration of \citet{ALVAREZESTEBAN2016744} on the shared-head Gram matrix.

\section{Formulation}
\label{sub:formulation}

\paragraph{Federated multivariate regression.}
We consider a federated multivariate regression problem with $M\ge2$ clients. Client $m\in\{1,\ldots,M\}$ owns
$\mathcal D_m=\{(x_{m,i},y_{m,i})\}_{i=1}^{N_m}$ with inputs $x_{m,i}$ and targets
$y_{m,i}\in\R^C$ ($C\ge2$, not assumed centered), collected columnwise as
$Y_m=[y_{m,1},\ldots,y_{m,N_m}]\in\R^{C\times N_m}$. Write $X_m=[x_{m,1},\ldots,x_{m,N_m}]$ for the
inputs, and let $N=\sum_mN_m$ and $p_m=N_m/N$.

\paragraph{Federated averaging (FedAvg) with private backbones and a shared head.}
A class of federated methods keeps each client's feature backbone local and averages only a shared
head~\citep{liang2020thinklocallyactglobally,10.1145/3545008.3545073}.
In our model, the shared head is the last linear layer $(W,b)$, with $W\in\R^{C\times P}$ and $b\in\R^C$, where $P$ is the
common feature dimension. Only $(W,b)$ is broadcast and averaged, while the backbone parameters $\theta_m$
stay local.
At each communication round $t$, the server broadcasts the current head $(W_t,b_t)$ to all clients.
Each client trains this head together with its private backbone on local data and uploads only the
resulting head $(\widetilde W_{m,t},\widetilde b_{m,t})$. The server then sets
$(W_{t+1},b_{t+1})=\sum_mp_m(\widetilde W_{m,t},\widetilde b_{m,t})$.

\paragraph{The unconstrained feature model.}
The UFM replaces the backbone features $H_m(X_m;\theta_m)\in\R^{P\times N_m}$ with free features
$H_m\in\R^{P\times N_m}$, giving the local objective
\begin{equation}
 \mathcal L_m(W,b,H_m)=\frac{1}{2N_m}\norm{WH_m+b\mathbf1^\top-Y_m}_F^2
 +\frac{\lambda_H}{2N_m}\norm{H_m}_F^2+\frac{\lambda_W}{2}\norm W_F^2,
 \label{sub:loss}
\end{equation}
with $\lambda_H,\lambda_W>0$; the bias is unregularized and the squared error sums over the $C$ outputs.
In the rounds above, each client now optimizes $H_m$ in place of $\theta_m$, and the server still averages
only $(W,b)$. The feature penalty is a common abstraction for standard penalties and normalizations on the backbone parameters, whose effects on the features can be complicated~\citep{Mixon2022,doi:10.1073/pnas.2103091118}. The centralized reference is another UFM with one head for the pooled data and minimizes
\begin{equation}
 \mathcal L_{\cen}(W,b,\{H_m\},\{Y_m\})=\sum_mp_m\mathcal L_m(W,b,H_m).
 \label{sub:pooled-loss}
\end{equation}
We compare the limit of the server's head Gram matrix $G_t=W_tW_t^\top$ with the Gram matrix of an
optimal head of this centralized objective.

\paragraph{Target moments and the fully active regime.}
Here we introduce and summarize quantities important for later analysis:
\begin{align}
&
\mu_m=\frac{Y_m\mathbf1}{N_m},
\quad
\mu_g=\sum_mp_m\mu_m,\\
&\Sigma_m=\frac{Y_mY_m^\top}{N_m}-\mu_m\mu_m^\top,
\quad
 \Sigma_{\mathrm{within}}=\sum_mp_m\Sigma_m,\quad
 \Sigma_\mu=\sum_mp_m(\mu_m-\mu_g)(\mu_m-\mu_g)^\top,\\
& \Sigma_{\cen}=\Sigma_{\mathrm{within}}+\Sigma_\mu.
 \label{sub:pooled-covariance}
\end{align}
For our theoretical analysis, we assume $P\ge C$ and
$\lambda_{\min}(\Sigma_m)>\lambda_H\lambda_W$ for every client (hence $N_m\ge C+1$). The same eigenvalue bound holds for $\Sigma_{\mathrm{within}}$ and $\Sigma_{\cen}$.
We refer to these conditions as the \emph{fully active regime}, under which every client's optimal head Gram matrix is positive definite.

\paragraph{Selecting locally optimal heads.}
In the fully active regime, the global minimizers of \eqref{sub:loss} share the same bias and head Gram matrix, while the head itself remains nonunique up to orthogonal transformations (Section~4.1)~\citep{andriopoulos2025neuralmultivariateregressionqualitative}. Optimality alone therefore does not determine the server update. To resolve this ambiguity, we add the proximal regularizer of FedProx~\citep{MLSYS2020_1f5fe839},
$\frac{\rho}{2}\|W-W_t\|_F^2 $,
to each client's objective. For a fixed broadcast \(W_t\) of full row rank, the heads of the global minimizers of the penalized objective converge, as \(\rho\downarrow0\), to the unique optimal head of the original objective that is closest to \(W_t\) in Frobenius norm (Lemma~1). In the next section, we analyze the dynamics of the shared head Gram matrix \(G_t\) induced by imposing this limiting procedure at every round, and also characterize the gap between its limit and the optimal head Gram matrix of centralized training.

\section{Theoretical analysis}
\label{sub:theory}

\subsection{The optimal parameters at each client}
\label{sub:client-optimum}
We first express the local optimization problem in terms of the head alone. Given a dataset $Y$ and fixed $W$, the minimization of the loss over $(b,H)$ gives
\[
 b^\star(W)=\mu,\qquad
 H^\star(W)=W^\top A^{-1}(Y-\mu\mathbf1^\top),
\]
where $\mu$ is the sample mean of $Y$ and $A=WW^\top+\lambda_HI_C$. Substituting these expressions into the loss yields the profiled objective
\begin{equation}
 F(W;T)=\frac{\lambda_H}{2}\Tr(TA^{-1})+\frac{\lambda_W}{2}\norm W_F^2,
 \label{sub:profile}
\end{equation}
where $T$ is the sample covariance of $Y$. This profiled objective is strictly convex with respect to the Gram matrix $G=WW^\top$, so the solution is unique and the corresponding stationary condition is $(G+\lambda_HI_C)^2=(\lambda_H/\lambda_W)T$. Hence, the solution is given by the following function:
\begin{equation}
\phi(T):=\alpha T^{1/2}-\lambda_HI_C,
 \quad \alpha=\sqrt{\frac{\lambda_H}{\lambda_W}}.
\end{equation}
This accords with the regression-UFM optimum of \citet{NEURIPS2024_e4748b6b} (Appendix~\ref{sub:profile-proof}).
Using this function $\phi$, the optimal head Gram matrices for the local and centralized objectives are written as  $G_m=\phi(\Sigma_m)$ and $G_{\cen}=\phi(\Sigma_{\cen})$, respectively. Similarly, $G_{\mathrm{within}}=\phi(\Sigma_{\mathrm{within}})$ is the optimal head Gram matrix of the averaged profiled objective $\sum_mp_mF(W;\Sigma_m)=F(W;\Sigma_{\mathrm{within}})$.

\subsection{Selecting the optimal head}
\label{sub:selection-rule}
As shown above, the local objective determines a unique optimal head Gram matrix $G_m$. However, since $F(W;T)$ depends on $W$ only through $WW^\top$, the optimal head remains nonunique up to right orthogonal transformations. We resolve this ambiguity by selecting the optimal head closest to the broadcast $W$:
\begin{equation}
 \Pi_m(W)=\operatorname*{arg\,min}_{U:UU^\top=G_m}\norm{U-W}_F^2,
 \qquad G=WW^\top\succ0.
 \label{sub:nearest-return}
\end{equation}
For any feasible $U$, the squared distance expands as
\[
 \norm{U-W}_F^2=\Tr G_m+\Tr G-2\Tr(UW^\top).
\]
Finding the closest optimal head therefore reduces to maximizing $\Tr(UW^\top)$, a rectangular Procrustes problem. Its minimum squared distance equals the squared BW distance between the two Gram matrices~\citep{BHATIA2019165}:
\[
d_{\mathrm{BW}}^2(G_{\mathrm{ref}},G_{\mathrm{tar}})
:= \Tr G_{\mathrm{ref}}+\Tr G_{\mathrm{tar}}
-2\Tr\!\left[(G_{\mathrm{ref}}^{1/2}G_{\mathrm{tar}}G_{\mathrm{ref}}^{1/2})^{1/2}\right].
\]
The following lemma gives the unique minimizing head and shows how it is selected by a vanishing head-only proximal penalty.
\begin{lemma}[Closest optimal head and proximal selection]
\label{sub:selection-main}
Let client $m$ be fully active and let the broadcast head $W$ satisfy $G=WW^\top\succ0$.
(i) The minimizer in \eqref{sub:nearest-return} is unique. It equals $\Pi_m(W)=T_m(G)W$ with
$T_m(G)=G^{-1/2}(G^{1/2}G_mG^{1/2})^{1/2}G^{-1/2}\succ0$, and
$\norm{\Pi_m(W)-W}_F^2=d_{\mathrm{BW}}^2(G,G_m)$.
(ii) For $\rho>0$, every global minimizer $(U_\rho,b_\rho,H_\rho)$ of
$\mathcal L_m(U,b,H)+\frac{\rho}{2}\norm{U-W}_F^2$ has $b_\rho=\mu_m$ and
$\mathcal L_m(U_\rho,b_\rho,H_\rho)-\min\mathcal L_m\le\frac{\rho}{2}\,d_{\mathrm{BW}}^2(G,G_m)$, and
$U_\rho\to\Pi_m(W)$ as $\rho\downarrow0$.
\end{lemma}
Part~(i) follows from the classical Procrustes characterization of the BW distance~\citep{BHATIA2019165}, and guarantees that each client has a unique optimal head closest to the common broadcast. By choosing these heads, the selection rule resolves the orthogonal ambiguity and makes the server update well defined. Part~(ii) shows that the head-only proximal penalty, of the form used in FedProx~\citep{MLSYS2020_1f5fe839}, selects $\Pi_m(W)$ as $\rho\downarrow0$, provided that the penalized objective is globally minimized for each $\rho>0$. We impose this limiting procedure at every round, before the server averages the returned heads. The proof is given in Appendix~\ref{sub:selection-proof}.

\subsection{Convergence to the BW barycenter}
\label{sub:bw-endpoint}
Under the selection rule, averaging the returned heads induces a closed recursion for the shared Gram matrix $G_t=W_tW_t^\top$. The following theorem shows that this recursion converges to the BW barycenter of the clients' optimal Gram matrices.
\begin{theorem}[Convergence to the Bures--Wasserstein barycenter]
\label{sub:oracle}
In the fully active regime of Section~\ref{sub:formulation}, let
$G_0=W_0W_0^\top\succ0$ and suppose
each client returns $(\Pi_m(W_t),\mu_m)$. Then every $G_t$ is positive definite,
$b_t=\mu_g$ for $t\ge1$, and the autonomous Gram dynamics are
\begin{equation}
 G_{t+1}=\mathcal G(G_t):=
 G_{t}^{-1/2}\left[\sum_mp_m(G_{t}^{1/2}G_mG_{t}^{1/2})^{1/2}\right]^2G_{t}^{-1/2}.
 \label{sub:gram-map}
\end{equation}
For every such initialization, $G_t$ converges to the unique positive-definite BW barycenter
$G_\star$, characterized by
\begin{equation}
 G_\star=\operatorname*{arg\,min}_{X\succ0}\sum_mp_m d_{\mathrm{BW}}^2(X,G_m),
 \qquad
 G_\star=\sum_mp_m(G_{\star}^{1/2}G_mG_{\star}^{1/2})^{1/2}.
 \label{sub:barycenter}
\end{equation}
\end{theorem}
\begin{proof}[Proof sketch]
By Lemma~\ref{sub:selection-main}, $\Pi_m(W_t)=T_m(G_t)W_t$, where each $T_m(G)$ is positive definite and
satisfies $T_m(G)\,G\,T_m(G)=G_m$.
All clients act on the same broadcast, so their average is
$W_{t+1}=S(G_t)W_t$, where $S(G)=\sum_mp_mT_m(G)\succ0$.
Consequently $G_{t+1}=S(G_t)G_tS(G_t)\succ0$.
Substituting the expression for $T_m$ gives \eqref{sub:gram-map}, while
averaging the local biases gives $\mu_g$.

The resulting Gram map is the barycenter iteration of
\citet{ALVAREZESTEBAN2016744}. Their convergence theorem applies because the clients'
optimal Gram matrices are positive definite, the weights are positive and sum to one, and
$G_0\succ0$. It yields the limit and characterization in
\eqref{sub:barycenter}. Appendices~\ref{sub:selection-proof}--\ref{sub:outer-proof}
give the projection proof and the full correspondence with that theorem.
\end{proof}

The limiting Gram matrix is therefore the BW barycenter of the local optimal Gram matrices $G_m=\phi(\Sigma_m)$. It depends only on the client covariances, aggregation weights,
and regularization parameters, and is independent of the choice of $G_0\succ0$ and the feature dimension $P\ge C$. Centralized training instead yields $G_{\cen}=\phi(\Sigma_{\cen})$, which generally differs from
this barycenter.

\subsection{The gap to centralized training and its correction}
\label{sub:consequences}
The barycenter fixed-point equation yields the following
decomposition of the gap $G_{\cen}-G_\star$ into three
positive-semidefinite terms.
\begin{proposition}[Three positive-semidefinite terms of the gap]
\label{sub:channels}
Under the assumptions of Theorem~\ref{sub:oracle}, choose any $W_\star$ satisfying $W_\star W_\star^\top=G_\star$ and
set $U_m^\star=\Pi_m(W_\star)$. Then $\sum_mp_mU_m^\star=W_\star$ and
\begin{equation}
 G_{\cen}-G_\star=\mathcal M_\mu+\mathcal M_\Sigma+\mathcal M_A.
\label{sub:three-channel}
\end{equation}
Here $\mathcal M_\mu=G_{\cen}-G_{\mathrm{within}}$,
$\mathcal M_\Sigma=G_{\mathrm{within}}-\sum_mp_mG_m$,
and $\mathcal M_A=\sum_mp_mG_m-G_\star$.
All three matrices are positive semidefinite (PSD). Hence
$G_\star\preceq\sum_mp_mG_m\preceq
G_{\mathrm{within}}\preceq G_{\cen}$.
The gap vanishes if and only if all client means agree and all client
covariances agree. Moreover,
\begin{equation}
 \Tr\mathcal M_A=\sum_mp_m d_{\mathrm{BW}}^2(G_\star,G_m).
 \label{sub:intrinsic-variance}
\end{equation}
\end{proposition}
\begin{proof}[Proof sketch]
The fixed-point equation implies $\sum_m p_m T_m(G_\star)=I_C$, and hence $\sum_m p_m U_m^\star=W_\star$. Therefore, $\sum_mp_mG_m-G_\star
=\sum_mp_m(U_m^\star-W_\star)(U_m^\star-W_\star)^\top$.
Add and subtract $G_{\mathrm{within}}$ and $\sum_mp_mG_m$
in $G_{\cen}-G_\star$ to obtain \eqref{sub:three-channel}.
Square-root monotonicity and concavity make $\mathcal M_\mu$ and $\mathcal M_\Sigma$ PSD, respectively; $\mathcal M_A$ is a covariance matrix.

A zero total gap forces every PSD term to vanish. The mean term then forces
the means to agree, and zero averaging covariance forces the returned heads,
hence the client covariances, to agree. The converse follows by substitution.
Taking traces and using the distance identity of Lemma~\ref{sub:selection-main} gives the BW
variance in \eqref{sub:intrinsic-variance}. Appendix~\ref{sub:channels-proof} provides these steps in
detail and bounds this variance by the pairwise BW distances between the $G_m$.
\end{proof}

The term $\mathcal M_\mu$ captures the scatter of client means
removed by clientwise centering, while $\mathcal M_\Sigma$
is the Jensen gap associated with the operator concavity of $\phi$.
The term $\mathcal M_A$ is the weighted scatter matrix of the
selected heads, and its trace equals the BW variance of the
clients' optimal Gram matrices.
Thus, resolving the orthogonal ambiguity does not by itself
eliminate the gap to centralized training.

To recover the centralized Gram matrix, we make every client's profiled objective equal to $F(W;\Sigma_{\cen})$.
Since $F(W;T)$ is affine in $T$, this can be achieved by adding the head-only moment correction:
\begin{equation}
 C_m(W)=\frac{\lambda_H}{2}\Tr\{(\Sigma_{\cen}-\Sigma_m)
                   (WW^\top+\lambda_HI_C)^{-1}\}.
 \label{sub:correction-term}
\end{equation}
It satisfies $F(W;\Sigma_m)+C_m(W)=F(W;\Sigma_{\cen})$ for every $W$.
Thus all corrected clients have the same optimal Gram matrix $G_{\cen}$.

\begin{proposition}[One-round recovery of the centralized Gram matrix]
\label{sub:correction}
In the fully active regime of Section~\ref{sub:formulation}, suppose
$W_tW_t^\top\succ0$. If every
client globally minimizes $\mathcal L_m+C_m$ and returns the optimal head closest to
$W_t$, the returned head weight matrices $\widetilde W_{m,t}$ coincide, so $G_{t+1}=G_{\cen}$, and the returned biases
$\mu_m$ average to $b_{t+1}=\mu_g$.
\end{proposition}
\begin{proof}[Proof sketch]
The corrected clients have the same optimal Gram matrix $G_{\cen}$ and receive the same broadcast.
Uniqueness of the closest optimal head makes their returned weight matrices $\widetilde W_{m,t}$ identical, while
their biases remain $\mu_m$. Averaging proves the claim.
Appendix~\ref{sub:correction-proof} also verifies existence of the corrected
global minimizers.
\end{proof}

The correction requires a one-time exchange of client means
and covariances, with $O(C^2)$ entries per client.
The local feature solutions still use clientwise-centered targets,
so matching the head Gram matrix and aggregate bias need not
reproduce centralized predictions
(Appendix~\ref{sub:correction-proof}).

\label{sub:body-end}
\section{Experiments}
\label{sec:experiments}
We first check the UFM predictions numerically, then test whether they describe head geometry and dynamics in trained networks.

\subsection{UFM numerical experiments}
\label{sec:ufm-numerical}
\begin{figure}[!t]
 \centering
 \includegraphics[width=\linewidth]{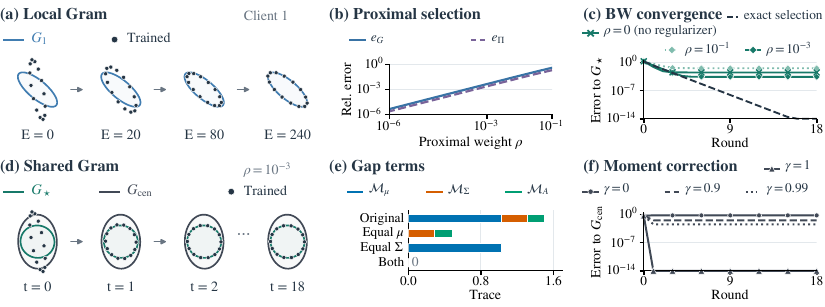}
 \caption{\textbf{UFM numerical checks on one fixed instance} ($C=P=2$, $M=3$, $\lambda_H=\lambda_W=0.1$; Appendix~\ref{sub:num-common}).
 Ellipses $\{G^{1/2}u:\norm u_2=1\}$ represent Gram matrices $G$, and dots trace trained ones.
 $E$ counts full-batch local steps, and $\gamma$ scales the moment correction.
 Errors below $10^{-14}$ are drawn at the axis floor.}
 \label{fig:ufm-common-instance}
\end{figure}

Figure~\ref{fig:ufm-common-instance} illustrates the theoretical results on one instance with two targets and three clients.
Panel (a) shows joint optimization of $(W,b,H)$ bringing the local Gram matrix toward $G_m$.
Panel (b) shows the relative errors to the optimal Gram ($e_G$) and the selected head ($e_\Pi$) decreasing as $\rho$ decreases, as predicted by Lemma~\ref{sub:selection-main}.

Panel (c) compares the shared Gram matrices obtained by exact selection
and by numerical optimization with and without the proximal penalty.
Exact selection converges to $G_\star$, as predicted by
Theorem~\ref{sub:oracle}.
With numerically optimized heads, the final shared Gram is closer to
$G_\star$ at $\rho=10^{-3}$ than without the penalty.
Panel (d) shows this $\rho=10^{-3}$ trajectory as ellipses:
the shared Gram approaches $G_\star$, while a gap to $G_{\cen}$ remains.
Appendix~\ref{sub:num-common} specifies the numerical objectives,
solvers and initializations.

Panel (e) checks the gap decomposition of Proposition~\ref{sub:channels}.
Equalizing client means removes $\mathcal M_\mu$, equalizing covariances removes $\mathcal M_\Sigma$ and $\mathcal M_A$, and equalizing both closes the gap.
Panel (f) applies $\gamma C_m$ under exact selection: full correction ($\gamma=1$) recovers $G_{\cen}$ in one round, as Proposition~\ref{sub:correction} predicts, whereas partial correction leaves a gap.
Appendix~\ref{sub:numerical} repeats these checks on random instances of several sizes.

\subsection{DNN experiments}
\label{sec:dnn}
\begin{figure}[!t]
 \centering
 \includegraphics[width=\linewidth]{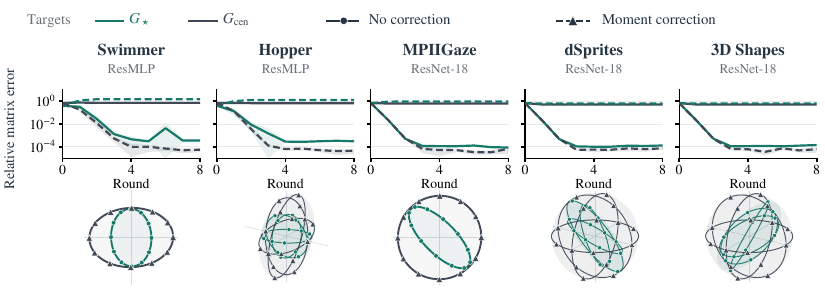}
 \caption{\textbf{Proximal training with and without moment correction} ($\rho=10^{-3}$). Colors indicate the reference Gram; solid and dashed curves denote training without and with correction. Curves show seed means $\pm1$ standard deviation. Below, outlines show $G_\star$ and $G_{\cen}$, and markers show final Grams for seed 0. Figure~\ref{fig:dnn-selection-appendix} gives the other five settings.}
 \label{fig:dnn-selection}
\end{figure}
We test whether trained networks approach $G_\star$ without correction and $G_{\cen}$ with correction, and whether the proximal penalty improves agreement with these predictions.

\subsubsection{Experimental setting}
\paragraph{Datasets and models.}
We use five tabular and five image regression datasets.
The tabular datasets are Beijing Multi-Site Air Quality~\citep{beijing_multi-site_air_quality_501,10.1098/rspa.2017.0457} and four MuJoCo tasks (Swimmer, Hopper, HalfCheetah and Walker; \citealp{6386109}) from JAT~\citep{gallouedec2024jacktradesmastersome}.
The image datasets are MPIIGaze~\citep{Zhang_2015_CVPR}, 300W-LP~\citep{Zhu_2016_CVPR}, dSprites~\citep{dsprites17}, 3D Shapes~\citep{3dshapes18} and CIFAR-geo, constructed from CIFAR\nobreakdash-10~\citep{Krizhevsky09learningmultiple}.
We partition the data to give clients different target means and covariances (Appendix~\ref{app:dnn-datasets}).
The backbones are ResMLP~\citep{NEURIPS2025_93da65de} for tabular data and ResNet-18~\citep{He_2016_CVPR} for images.

\paragraph{Training conditions.}
Each client trains its head and private backbone using minibatch AdamW on \eqref{sub:loss}.
Here $H_m$ contains the backbone outputs, whose squared norms are penalized as in the UFM.
We use $E=4000$ local epochs per round to approach local optimality and average only the heads over eight rounds.
Ordinary training uses $\rho=0$; proximal training adds $\rho\|W-W_t\|_F^2/2$ with $\rho=10^{-3}$ to approximate the selection rule.
Each is run with and without $C_m$, giving four procedures, with five seeds $\{0,1,2,3,4\}$ per dataset (Appendix~\ref{app:exp-overview}).

\paragraph{Evaluation.}
We compute $G_\star$ and $G_{\cen}$ from training-target moments, client weights and regularization, independently of the trained networks.
Relative matrix error $e_F(G,G')=\|G-G'\|_F/\|G'\|_F$ measures agreement in scale and shape.
Endpoint errors compare $G_8$ with $G_\star$ without correction and with $G_{\cen}$ under correction.
For uncorrected runs, trajectory errors instead compare with \eqref{sub:gram-map} from the same initial Gram (Appendix~\ref{app:anchor-dnn}).

\subsubsection{Experimental results}
\label{sec:dnn-results}
\paragraph{Convergence toward the BW prediction.}
\label{sec:dnn-selection}
Without correction, proximal training approaches $G_\star$ while remaining far from $G_{\cen}$ (Figure~\ref{fig:dnn-selection}).
Across the 50 dataset--seed combinations, the median endpoint error to $G_\star$ is $2.90\times10^{-4}$, compared with $1.63\times10^{-2}$ for ordinary training.
The proximal penalty improves endpoint agreement in every pair, and the proximal trajectories closely follow the predicted Gram iteration (Appendix~\ref{app:anchor-dnn}).

\paragraph{Moment correction.}
\label{sec:dnn-correction}
With correction, the shared Gram instead approaches $G_{\cen}$ (Figure~\ref{fig:dnn-selection}).
Across the same 50 combinations, the median endpoint error is $7.08\times10^{-5}$ with the proximal penalty and $5.06\times10^{-3}$ without it (Appendix~\ref{app:correction-rounds}).
Figure~\ref{fig:dnn-mpiigaze} illustrates these different endpoints and the improvement from the proximal penalty on MPIIGaze.

\begin{figure}[!t]
 \centering
 \includegraphics[width=\linewidth]{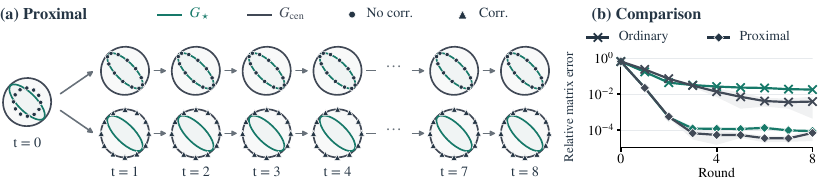}
 \caption{\textbf{Gram evolution on MPIIGaze.} (a) Seed-0 Gram matrices under proximal training. (b) Relative errors to $G_\star$ without correction (green) and to $G_{\cen}$ with correction (graphite gray), with five-seed means $\pm1$ standard deviation.}
 \label{fig:dnn-mpiigaze}
\end{figure}

\paragraph{Additional results.}
Clientwise target centering removes the mean term but leaves a gap, consistent with Proposition~\ref{sub:channels} (Appendix~\ref{app:moment-dnn}).
Proximal training agrees best with $G_\star$ at an intermediate penalty weight (Appendix~\ref{app:proximal-weight}).
Aligning trained heads to the broadcast also improves agreement with $G_\star$ and lowers the objective after averaging compared with the same unaligned heads (Appendix~\ref{app:alignment}).
The correspondence persists across other architectures, larger training sets and, with sufficient local training, full-model averaging (Appendix~\ref{app:dnn-scope}).

\section{Discussion}
\label{sec:discussion}

\paragraph{Limitations.}
Our main analysis assumes that clients reach their local optima
in every round, unlike typical federated training.
It also assumes the fully active regime and does not cover
rank-deficient optimal head Gram matrices.
The UFM does not model features of unseen inputs, so our analysis
provides no test-error guarantees.

\paragraph{Perspectives.}
To partially resolve the above limitations, we performed additional experiments on test
error and an analysis of finite local training.
Compared with ordinary training, proximal training yields comparable
test error on the image datasets, while proximal training with
moment correction yields a modest median reduction
(Appendix~\ref{app:prediction}).
For finite local training, we studied profiled updates, which minimize
features and bias exactly before each head step.
For small total step length $\eta E$, we established a locally
attracting Gram fixed point near $G_{\mathrm{within}}$, the fixed
point for one local step per round, and quantify its first-order
displacement (Appendix~\ref{sub:supplementary-analysis}).
Numerical checks support the small-step predictions; a separate
sweep over local step counts shows a gradual shift toward the BW
barycenter (Appendix~\ref{sub:num-finite}).

\paragraph{Conclusion.}
Under the proposed selection rule, we characterized the UFM Gram dynamics
through the BW barycenter, decomposed the gap to centralized
training, and recovered the centralized optimal Gram in one round
by moment correction.
UFM and neural-network experiments support the predicted Gram
behavior.

\newpage
\subsection*{AI use statement}
In this work, we used generative AI tools to assist with translation and to suggest ideas for the proofs of our mathematical claims.
Additionally, we used them to suggest improvements to the English writing, to check and adjust the \LaTeX{} formatting, to search for related work beyond the studies already known to us so that the paper is positioned accurately, to search for related theoretical results, and to review and organize code.
The authors reviewed and approved all revised text and verified all core theoretical results presented in this paper.
We take responsibility for the final content of this work, including text, claims or artifacts produced with the aid of generative AI.

\subsection*{Reproducibility statement}
\paragraph{Theory.}
Section~\ref{sub:formulation} defines the federated UFM, including the local and centralized objectives and the fully active conditions assumed by every result, and Section~\ref{sub:theory} states each result with its assumptions.
The appendices give complete proofs: Appendix~\ref{sub:bw-proof} derives the optimal head Gram matrix and proves Lemma~\ref{sub:selection-main} and Theorem~\ref{sub:oracle}, Appendix~\ref{sub:channels-proof} proves Proposition~\ref{sub:channels}, and Appendix~\ref{sub:correction-proof} proves Proposition~\ref{sub:correction}.
Appendix~\ref{sub:supplementary-analysis} states and proves Theorem~\ref{sub:finite} on finite local steps.

\paragraph{Experiments.}
Appendix~\ref{sub:numerical} describes how the UFM instances are generated and solved, gives the fixed instance of Figure~\ref{fig:ufm-common-instance} with its random seed, and repeats each numerical check on many random instances of several sizes.
For the DNN experiments, Appendix~\ref{app:exp-details} describes the datasets with their preprocessing, client partitions and licenses (Appendix~\ref{app:dnn-datasets}), the architectures, the federated training procedure with all hyperparameters (Appendix~\ref{app:return-protocol}), the evaluation (Appendix~\ref{app:local-screen}), and all runs together with the compute used (Appendix~\ref{app:exp-overview}).
The fully active condition holds in every reported DNN experiment (Appendix~\ref{app:return-protocol}), and the predictions $G_\star$ and $G_{\cen}$ are computed from the training-target moments, client weights and regularization parameters alone, independently of the trained networks.
Each DNN experiment is repeated over random seeds: five seeds $\{0,1,2,3,4\}$ for the main experiments and at least the three seeds $\{0,1,2\}$ for each supplementary experiment (Appendix~\ref{app:exp-overview}).
We report seed means with standard-deviation bands, medians over all dataset--seed combinations, and paired comparisons within each combination; Appendix~\ref{app:supp-dnn} reports the results for every setting.
\bibliography{refs}
\bibliographystyle{iclr2027_conference}
\appendix
\clearpage
\section*{Appendix contents}
Appendices~\ref{sub:bw-proof}--\ref{sub:correction-proof} give the proofs of
the main results. Appendices~\ref{sub:numerical}--\ref{app:supp-dnn}
report the UFM and DNN experiments. Appendix~\ref{sub:supplementary-analysis}
collects supplementary analyses and their numerical checks.

\makeatletter
\let\appendix@orig@addcontentsline\addcontentsline
\begingroup
\setcounter{tocdepth}{2}
\setlength{\parskip}{0pt}
\@starttoc{apc}
\endgroup
\def\addcontentsline#1#2#3{%
  \def\appendix@destination{#1}%
  \def\appendix@toc{toc}%
  \ifx\appendix@destination\appendix@toc
    \addtocontents{apc}{\protect\contentsline{#2}{#3}{\thepage}{\@currentHref}%
      \protected@file@percent}%
  \fi
  \appendix@orig@addcontentsline{#1}{#2}{#3}%
}
\makeatother
\clearpage

\section{From local optima to the BW limit}
\label{sub:bw-proof}
This appendix gives the proofs for
Sections~\ref{sub:client-optimum}--\ref{sub:bw-endpoint}.
We use the fully active assumptions of Section~\ref{sub:formulation}, together with each result's additional hypotheses.

\subsection{Profiling and the optimal head Gram matrix}
\label{sub:profile-proof}
We derive the profile and the local and centralized optimal Gram matrices
used in Section~\ref{sub:client-optimum}.
For a target matrix $Y\in\R^{C\times N}$, write
\[
 \mu=Y\mathbf1/N,\qquad
 \widetilde Y=Y-\mu\mathbf1^\top,\qquad
 T=\widetilde Y\widetilde Y^\top/N,\qquad
 A=WW^\top+\lambda_HI_C,
\]
where $\mathbf1\in\R^N$ is the all-ones vector. At fixed $(W,b)$, the minimizer of \eqref{sub:loss} with respect to $H$ is unique thanks to the strict
convexity and takes
\[
 H(W,b)=(W^\top W+\lambda_HI_P)^{-1}W^\top(Y-b\mathbf1^\top)
       =W^\top A^{-1}(Y-b\mathbf1^\top).
\]
Substitution into the loss gives
\[
 \frac{\lambda_H}{2}\Tr(TA^{-1})
 +\frac{\lambda_H}{2}(b-\mu)^\top A^{-1}(b-\mu)
 +\frac{\lambda_W}{2}\norm W_F^2.
\]
Since $A\succ0$, the unique minimizing bias is $b=\mu$, giving
$H^\star(W)=W^\top A^{-1}\widetilde Y$ and the profile \eqref{sub:profile}.

Writing $G=WW^\top$, we minimize
\[
 f_T(G)=\frac{\lambda_H}{2}\Tr[T(G+\lambda_HI_C)^{-1}]
       +\frac{\lambda_W}{2}\Tr G
\]
over $G\succeq0$, all representable since $P\ge C$. For $X,T\succ0$ and
nonzero symmetric $D$,
\[
 \begin{gathered}
 \left.\frac{d^2}{dt^2}\Tr[T(X+tD)^{-1}]\right|_{t=0}=2\Tr(KMK)>0,\\
 K=X^{-1/2}DX^{-1/2},\qquad M=X^{-1/2}TX^{-1/2}\succ0.
 \end{gathered}
\]
Hence $f_T$ is strictly convex. Its interior stationary equation is
\[
 -\frac{\lambda_H}{2}(G+\lambda_HI_C)^{-1}T(G+\lambda_HI_C)^{-1}
 +\frac{\lambda_W}{2}I_C=0,
 \qquad (G+\lambda_HI_C)^2=\frac{\lambda_H}{\lambda_W}T.
\]
Under the fully active condition
$\lambda_{\min}(T)>\lambda_H\lambda_W$, its solution is
$G=\phi(T)=\sqrt{\lambda_H/\lambda_W}\,T^{1/2}-\lambda_HI_C\succ0$. For any $W$ satisfying $WW^\top=\phi(T)$, the triple $(W,\mu,H^\star(W))$ is a global minimizer of the full loss \citep{NEURIPS2024_e4748b6b}.

With $p_m=N_m/N$, \eqref{sub:pooled-loss} is the same loss on concatenated
client data, whose moments are $(\mu_g,\Sigma_{\cen})$. Thus the local
and centralized optimal head Gram matrices are $G_m=\phi(\Sigma_m)$ and
$G_{\cen}=\phi(\Sigma_{\cen})$.

\subsection{Unique selected head weights and the proximal limit}
\label{sub:selection-proof}
The following lemma establishes the selection rule of
Section~\ref{sub:selection-rule}. Taking $G_{\mathrm{tar}}=G_m$ and $T=\Sigma_m$
proves Lemma~\ref{sub:selection-main}; its loss bound follows from
\eqref{sub:selection-inequality}.
\begin{lemma}[Positive-definite rectangular projection and selection]
\label{sub:selection}
Let $G_{\mathrm{tar}}\succ0$ and let $W\in\R^{C\times P}$ have full row rank, with
$G=WW^\top\succ0$ and $P\ge C$.
The unique minimizer of $\norm{U-W}_F^2$ subject to $UU^\top=G_{\mathrm{tar}}$ is
\begin{equation}
 \Pi_{G_{\mathrm{tar}}}(W)=T_{G_{\mathrm{tar}}}(G)W,\qquad
 T_{G_{\mathrm{tar}}}(G)=G^{-1/2}(G^{1/2}G_{\mathrm{tar}}G^{1/2})^{1/2}G^{-1/2},
 \label{sub:projection}
\end{equation}
and its attained squared distance is $d_{\mathrm{BW}}^2(G,G_{\mathrm{tar}})$.

For the proximal-selection statement, let
$Y\in\mathbb{R}^{C\times N}$ have mean $\mu$ and covariance $T$
as defined in Appendix~\ref{sub:profile-proof}, and let $F(\cdot;T)$ be given by~\eqref{sub:profile},
with $\lambda_H,\lambda_W>0$.
Assume that $\lambda_{\min}(T)>\lambda_H\lambda_W$ and
$G_{\mathrm{tar}}=\phi(T)$. For the fixed reference $W$, global minimizers of
\[
F(U;T)+\frac{\rho}{2}\|U-W\|_F^2
\]
exist for every $\rho>0$. Any choice of these minimizers
$U_\rho$ satisfies
\[
U_\rho\to\Pi_{G_{\mathrm{tar}}}(W)
\qquad\text{as }\rho\downarrow0.
\]
For the full UFM loss (1) formed from $Y$ with the same
head-only penalty, every global minimizer
$(U_\rho,b_\rho,H_\rho)$ satisfies
$b_\rho=\mu$ and $H_\rho=H^\star(U_\rho)$,
where $H^\star$ is defined in Appendix~\ref{sub:profile-proof}.
\end{lemma}

\begin{proof}
Write $U=G_{\mathrm{tar}}^{1/2}R$, where $RR^\top=I_C$. The squared distance differs by a
constant from $-2\Tr(RZ)$, where $Z=W^\top G_{\mathrm{tar}}^{1/2}$ has full column rank.
Let $Z=\Phi D\Psi^\top$ be its thin singular value decomposition; all diagonal
entries of $D$ are positive. Then $\Tr(RZ)\le\Tr D$.
Equality forces every diagonal entry of the contraction
$\Psi^\top R\Phi$ to equal one. Writing $\phi_j$ and $\psi_j$ for the columns of $\Phi$ and $\Psi$, this means $\psi_j^\top R\phi_j=1$.
Since $\norm{R\phi_j}\le1$ and $\norm{\psi_j}=1$, equality in
Cauchy--Schwarz gives $R\phi_j=\psi_j$ for every $j$, hence $R\Phi=\Psi$.
Complete $\Phi$ to an orthogonal matrix $[\Phi,\Phi_\perp]$ and write
$R=\Psi\Phi^\top+Q\Phi_\perp^\top$.
The row-orthonormality constraint gives $I_C+QQ^\top=I_C$, so $Q=0$. Thus the optimizer is unique.

Since $(T_{G_{\mathrm{tar}}}(G) W)(T_{G_{\mathrm{tar}}}(G) W)^{\top }=T_{G_{\mathrm{tar}}}(G)\,G\,T_{G_{\mathrm{tar}}}(G)=G_{\mathrm{tar}}$ holds,
$U=T_{G_{\mathrm{tar}}}(G)W$ is feasible, and attains the maximum:
\[
 \Tr[UW^\top]= \Tr[T_{G_{\mathrm{tar}}}(G)G]=\Tr[ (G^{1/2}G_{\mathrm{tar}}G^{1/2})^{1/2}]=\Tr D.
\]

To justify the last equality, set $X=G^{1/2}G_{\mathrm{tar}}^{1/2}$. Then
\[
X^\top X
=G_{\mathrm{tar}}^{1/2}GG_{\mathrm{tar}}^{1/2}
=Z^\top Z
=\Psi D^2\Psi^\top,
\qquad
XX^\top=G^{1/2}G_{\mathrm{tar}}G^{1/2}.
\]
Since $XX^\top$ and $X^\top X$ have the same eigenvalues,
the eigenvalues of $G^{1/2}G_{\mathrm{tar}}G^{1/2}$ are the squared singular
values of $Z$. Taking the positive square root therefore gives
\[
\Tr\!\left[(G^{1/2}G_{\mathrm{tar}}G^{1/2})^{1/2}\right]
=\sum_{i=1}^C D_{ii}
=\Tr D.
\]
Hence $T_{G_{\mathrm{tar}}}(G)W$ attains the trace upper bound and thus is the unique solution discussed above. Substituting this maximum into $\norm{U-W}_F^2=\Tr G_{\mathrm{tar}}+\Tr G-2\Tr(UW^\top)$ gives the BW distance formula.

For the proximal limit, put $F_\star=\min_U F(U;T)$ and
$U^\dagger=\Pi_{G_{\mathrm{tar}}}(W)$. Coercivity from $\lambda_W\norm U_F^2/2$ ensures
a global minimizer $U_\rho$ for each $\rho>0$. Comparison with $U^\dagger$ gives
\begin{equation}
 0\le F(U_\rho;T)-F_\star,
 \qquad
 F(U_\rho;T)-F_\star+\frac\rho2\norm{U_\rho-W}_F^2
 \le\frac\rho2\norm{U^\dagger-W}_F^2.
 \label{sub:selection-inequality}
\end{equation}
Thus $\norm{U_\rho-W}_F\le\norm{U^\dagger-W}_F$ and
$F(U_\rho;T)\to F_\star$. The heads $U_\rho$ thus lie in a fixed compact ball centered at $W$, so every sequence of $U_\rho$ with $\rho\downarrow 0$ has a convergent subsequence by the Bolzano--Weierstrass theorem. By continuity, every such subsequential limit minimizes $F(\cdot;T)$ and is no farther from $W$ than $U^\dagger$. Since $U^\dagger$ is the unique head closest to $W$ among the global minimizers of $F(\cdot;T)$, every such limit equals $U^\dagger$.

Since the proximal penalty depends only on the head, the conditional bias and feature solutions remain unchanged. Thus every global minimizer of the penalized full objective satisfies
$b_\rho=\mu$ and $H_\rho=H^\star(U_\rho)$, as derived in Appendix~\ref{sub:profile-proof}, and its original loss equals $F(U_\rho;T)$. Consequently, \eqref{sub:selection-inequality} gives
\[
0 \le
\mathcal L(U_\rho,b_\rho,H_\rho)-\min\mathcal L
= F(U_\rho;T)-F_\star
\le \frac{\rho}{2}\|U^\dagger-W\|_F^2.
\]
The excess loss therefore tends to zero as $\rho\downarrow0$.
\end{proof}

\subsection{The closed Gram recursion and its BW limit}
\label{sub:outer-proof}
We prove Theorem~\ref{sub:oracle} by deriving the Gram recursion and checking the hypotheses of the BW barycenter convergence theorem. Set $T_m(G):=T_{G_m}(G)$ and $S(G):=\sum_mp_mT_m(G)$. By
Lemma~\ref{sub:selection}, we have the head update $W_{t+1}=\sum_mp_m\Pi_m(W_t)=S(G_t)W_t$, implying the following closed Gram recursion:
\[
G_{t+1}=S(G_t)G_tS(G_t).
\]
Substituting $T_m$ gives \eqref{sub:gram-map}. Starting from $G_0\succ0$, induction gives $T_m(G_t)\succ0$,
$S(G_t)\succ0$ and $G_{t+1}\succ0$. The returned biases $\mu_m$ average to $b_{t+1}=\mu_g$.

\paragraph{Convergence of the Gram recursion.}
For positive weights summing to one and positive-definite $G_m$, the iteration \eqref{sub:gram-map} converges from every $G_0\succ0$ to the unique positive-definite BW barycenter satisfying \eqref{sub:barycenter} \citep[Theorem~4.2 and Remark~4.3]{ALVAREZESTEBAN2016744}.
Its variational characterization is the Gaussian barycenter characterization
\citep{doi:10.1137/100805741,BHATIA2019165}.
Full activity supplies $G_m\succ0$, so this proves Theorem~\ref{sub:oracle}.

If $G_1,\ldots,G_M$ commute with each other, $(\sum_mp_mG_m^{1/2})^2$ satisfies the fixed-point
equation in their common eigenbasis, so uniqueness identifies it with $G_\star$.

\section{The gap to centralized training}
\label{sub:channels-proof}
This appendix proves the gap characterization in
Proposition~\ref{sub:channels} of Section~\ref{sub:consequences}.
We establish the PSD decomposition, characterize equality, and relate the
averaging term to BW distances.

\subsection{The three positive-semidefinite terms}
\label{sub:gap-decomposition-proof}
The weighted scatter of the selected head weights supplies the averaging
term in \eqref{sub:three-channel}.
The fixed-point equation in \eqref{sub:barycenter}, conjugated by
$G_{\star}^{-1/2}$, gives $S(G_\star)=I_C$. Hence any factor $W_\star$ of
$G_\star$ satisfies $\sum_mp_m\Pi_m(W_\star)=W_\star$.
More generally, the weighted scatter identity at any round gives
\begin{equation}
 \sum_mp_mG_m-G_{t+1}
 =\sum_mp_m(\Pi_m(W_t)-W_{t+1})(\Pi_m(W_t)-W_{t+1})^\top\succeq0.
 \label{sub:variance-identity}
\end{equation}
Expand the weighted scatter using $\sum_mp_m\Pi_m(W_t)=W_{t+1}$
and $\Pi_m(W_t)\Pi_m(W_t)^\top=G_m$ to obtain this identity.
At any factor $W_\star$ of $G_\star$, this scatter is
$\mathcal M_A=\sum_mp_mG_m-G_\star$, independently of that factor;
convergence of $W_t$ is not required.

Adding and subtracting $G_{\mathrm{within}}$ and
$\sum_mp_mG_m$ gives \eqref{sub:three-channel}. For
$\alpha=\sqrt{\lambda_H/\lambda_W}$,
\[
 \mathcal M_\mu=\alpha(\Sigma_{\cen}^{1/2}-\Sigma_{\mathrm{within}}^{1/2})\succeq0,
 \qquad
 \mathcal M_\Sigma=\alpha\left(\Sigma_{\mathrm{within}}^{1/2}
                                  -\sum_mp_m\Sigma_m^{1/2}\right)\succeq0.
\]
Since $\Sigma_{\cen}=\Sigma_{\mathrm{within}}+\Sigma_\mu\succeq
\Sigma_{\mathrm{within}}$, square-root monotonicity gives
$\mathcal M_\mu\succeq0$. Operator concavity of $\phi$ gives its Jensen
gap $\mathcal M_\Sigma\succeq0$
\citep[Theorems~V.1.9 and~V.2.5]{Bhatia1997}.
Equation~\eqref{sub:variance-identity} makes $\mathcal M_A$ the PSD weighted
scatter of $U_m^\star=\Pi_m(W_\star)$ around $W_\star$. This proves the
Loewner chain of Proposition~\ref{sub:channels}.

\subsection{The exact zero set}
\label{sub:gap-zero-proof}
We next prove the equality condition in Proposition~\ref{sub:channels}:
the gap vanishes exactly when client means and covariances agree.
A zero sum forces each PSD term to vanish. Square-root injectivity gives
$\mathcal M_\mu=0\iff\Sigma_\mu=0$. Since
$\Tr\Sigma_\mu=\sum_mp_m\norm{\mu_m-\mu_g}^2$ with $p_m>0$, all means
then agree. Likewise $\mathcal M_A=0$ gives
$\sum_mp_m\norm{U_m^\star-W_\star}_F^2=0$, so all $U_m^\star$, hence all
$G_m$, agree. Injectivity of $\phi$ implies equal covariances. Conversely,
equal means and covariances give $\Sigma_\mu=0$ and
$G_m=G_{\cen}=G_\star$, making all three terms zero.

\subsection{BW variance and pairwise bounds}
\label{sub:gap-variance-proof}
Choose any factor $W_\star W_\star^\top=G_\star$ and set
$U_m^\star=\Pi_m(W_\star)$. Taking the trace of
\[
\mathcal M_A
=\sum_m p_m(U_m^\star-W_\star)(U_m^\star-W_\star)^\top
\]
and using $\Tr(AA^\top)=\norm{A}_F^2$, we obtain
\[
\Tr\mathcal M_A
=\sum_m p_m\norm{U_m^\star-W_\star}_F^2.
\]
Each $U_m^\star$ is the factor of $G_m$ closest to $W_\star$.
Hence Lemma~\ref{sub:selection} gives
\[
\norm{U_m^\star-W_\star}_F^2
=d_{\mathrm{BW}}^2(G_\star,G_m).
\]
Substitution proves \eqref{sub:intrinsic-variance}.

In addition to this, we prove the following pairwise bounds mentioned in Section~\ref{sub:consequences}:
\begin{equation}
 D\le\Tr\mathcal M_A\le2D,\qquad
 D:=\sum_{m<n}p_mp_n d_{\mathrm{BW}}^2(G_m,G_n).
 \label{sub:pairwise-bounds}
\end{equation}
For the lower bound, we first express the scatter around
$W_\star$ in terms of pairwise distances.
Expanding the squared distances and using
$\sum_m p_mU_m^\star=W_\star$ gives
\[
\begin{aligned}
\sum_{m<n}p_mp_n\norm{U_m^\star-U_n^\star}_F^2
&=\frac12\sum_{m,n}p_mp_n
  \norm{U_m^\star-U_n^\star}_F^2\\
&=\sum_m p_m\norm{U_m^\star}_F^2
  -\norm{W_\star}_F^2\\
&=\Tr\mathcal M_A.
\end{aligned}
\]
For each pair $(m,n)$, Lemma~\ref{sub:selection} gives
\[
d_{\mathrm{BW}}^2(G_m,G_n)
=\min_{UU^\top=G_m}\norm{U-U_n^\star}_F^2.
\]
Since $U_m^\star U_m^{\star\top}=G_m$, the head $U_m^\star$
is a feasible choice in this minimization. Therefore,
\[
d_{\mathrm{BW}}^2(G_m,G_n)
\le \norm{U_m^\star-U_n^\star}_F^2.
\]
Multiplying by $p_mp_n$ and summing over $m<n$ yields
$D\le\Tr\mathcal M_A$.

For the upper bound, define
\[
V(X):=\sum_m p_m d_{\mathrm{BW}}^2(X,G_m).
\]
By the barycenter characterization in \eqref{sub:barycenter},
$G_\star$ is a global minimizer of $V$.
Since each $G_n$ is positive definite and therefore feasible,
we have $V(G_\star)\le V(G_n)$ for every $n$. Averaging these inequalities with weights $p_n$ gives
\[
\Tr\mathcal M_A
=V(G_\star)
\le\sum_n p_nV(G_n)
=\sum_{n,m}p_np_m d_{\mathrm{BW}}^2(G_n,G_m)
=2D.
\]
The last equality follows because the terms with $m=n$
are zero, while each pair with $m<n$ is counted twice.

\section{Moment correction and the scope of recovery}
\label{sub:correction-proof}
This appendix proves the recovery result in Proposition~\ref{sub:correction}
of Section~\ref{sub:consequences} and explains its scope. We identify the
moments needed for the correction and compare the resulting predictions
with centralized training.

\subsection{A common corrected profile and recovery}
\label{sub:correction-recovery-proof}
Since $C_m(W)$ depends only on $W$, minimizing over the bias
and features for a fixed $W$ gives the same conditional
solutions as in Appendix~\ref{sub:profile-proof}:
\[
b_m^\star(W)=\mu_m,
\qquad
H_m^\star(W)
=W^\top A^{-1}(Y_m-\mu_m1^\top),
\qquad
A=WW^\top+\lambda_HI_C.
\]
Substituting these solutions into the corrected objective gives
\[
\begin{aligned}
\min_{b,H}\{\mathcal L_m(W,b,H)+C_m(W)\}
&=F(W;\Sigma_m)+C_m(W)\\
&=\frac{\lambda_H}{2}
  \Tr\!\left[(\Sigma_m+\Sigma_{\cen}-\Sigma_m)A^{-1}\right]
  +\frac{\lambda_W}{2}\norm{W}_F^2\\
&=F(W;\Sigma_{\cen}).
\end{aligned}
\]
Thus every client has the same profiled objective for the head.

By Appendix~\ref{sub:profile-proof}, this common objective is minimized exactly
by the heads satisfying
\[
WW^\top=G_{\cen}
=\phi(\Sigma_{\cen})\succ0.
\]
Such heads exist because $P\ge C$.
Pairing any such head with the conditional bias and feature
solutions above therefore gives a global minimizer of each
client's corrected full objective.

Now fix a common broadcast $W_t$ of full row rank.
Under the selection rule, each client chooses the head
closest to $W_t$ among the factors of $G_{\cen}$.
By Lemma~\ref{sub:selection}, this head is unique, so
\[
\widetilde W_{m,t}=\Pi_{G_{\cen}}(W_t),
\qquad
\widetilde b_{m,t}=\mu_m
\quad\text{for every }m.
\]
Since $\sum_m p_m=1$, server averaging gives
\[
\begin{aligned}
W_{t+1}
&=\sum_m p_m\widetilde W_{m,t}
=\Pi_{G_{\cen}}(W_t),\\
G_{t+1}
&=W_{t+1}W_{t+1}^\top=G_{\cen},\\
b_{t+1}
&=\sum_m p_m\mu_m=\mu_g.
\end{aligned}
\]
This proves Proposition~\ref{sub:correction}.

Finally, Appendix~\ref{sub:selection-proof} applies to
the common profiled objective: any choice of global minimizers
\[
U_\rho\in\operatorname*{arg\,min}_{U}
\left\{
F(U;\Sigma_{\cen})
+\frac{\rho}{2}\norm{U-W_t}_F^2
\right\}
\]
satisfies $U_\rho\to\Pi_{G_{\cen}}(W_t)$ as $\rho\downarrow0$.
Thus the same selection can be imposed by globally minimizing
the penalized objective for each $\rho>0$ and taking this
limit before server averaging.

\subsection{The two correction terms and required moments}
\label{sub:correction-moments-proof}
Using $\Sigma_{\cen}=\Sigma_{\mathrm{within}}+\Sigma_\mu$,
we write
\[
C_m(W)
=\frac{\lambda_H}{2}\Tr(\Sigma_\mu A^{-1})
+\frac{\lambda_H}{2}
 \Tr\!\left[(\Sigma_{\mathrm{within}}-\Sigma_m)A^{-1}\right],
\qquad A=WW^\top+\lambda_HI_C.
\]
The first term restores the mean-scatter contribution removed
by local centering. The second replaces $\Sigma_m$ by
$\Sigma_{\mathrm{within}}$ in each profiled objective;
its weighted average is zero when evaluated at a common $W$.

Each client uploads $N_m$, $\mu_m$, and $\Sigma_m$ once.
With $N=\sum_m N_m$, the server computes
\[
\mu_g=\frac1N\sum_m N_m\mu_m,
\qquad
\Sigma_{\cen}
=\frac1N\sum_m N_m(\Sigma_m+\mu_m\mu_m^\top)
-\mu_g\mu_g^\top,
\]
and broadcasts $\Sigma_{\cen}$ so that each client can
evaluate $C_m(W)$.
This requires $O(C^2)$ additional entries per client,
or $O(MC^2)$ in total.

\subsection{Predictions with local and shared biases}
\label{sub:correction-predictions-proof}
Fix a common head $W$ with $G=WW^\top=G_{\cen}$, and let
$A=G+\lambda_HI_C$ and $d_m=\mu_m-\mu_g$.
The corrected client centers its targets using $\mu_m$,
whereas centralized training uses $\mu_g$.
Their optimal features on client $m$'s samples are therefore
\[
\begin{aligned}
H_m^{\mathrm{corr}}
&=W^\top A^{-1}(Y_m-\mu_m\mathbf1^\top),\\
H_m^{\cen}
&=W^\top A^{-1}(Y_m-\mu_g\mathbf1^\top).
\end{aligned}
\]
Hence
\begin{equation}
H_m^{\mathrm{corr}}-H_m^{\cen}
=-W^\top A^{-1}d_m\mathbf1^\top.
\label{sub:feature-offset}
\end{equation}

After aggregation, each client uses the shared bias $\mu_g$
while retaining its locally optimized features
$H_m^{\mathrm{corr}}$.
By \eqref{sub:feature-offset}, its prediction difference
from centralized training is
\[
\begin{aligned}
(WH_m^{\mathrm{corr}}+\mu_g\mathbf1^\top)
-(WH_m^{\cen}+\mu_g\mathbf1^\top)
&=-GA^{-1}d_m\mathbf1^\top.
\end{aligned}
\]
If the client instead retains its local bias $\mu_m$, the
prediction difference becomes
\[
\begin{aligned}
(WH_m^{\mathrm{corr}}+\mu_m\mathbf1^\top)
-(WH_m^{\cen}+\mu_g\mathbf1^\top)
&=\lambda_HA^{-1}d_m\mathbf1^\top.
\end{aligned}
\]
The latter equality uses $I_C-GA^{-1}=\lambda_HA^{-1}$.
Both prediction differences are nonzero whenever $\mu_m\ne\mu_g$.
Thus the correction recovers the centralized head Gram matrix
and aggregate bias, but need not reproduce centralized
training predictions.

\section{UFM numerical experiments}
\label{sub:numerical}
This appendix checks the results of Section~\ref{sub:theory} numerically
under the fully active conditions of Section~\ref{sub:formulation}.
Appendix~\ref{sub:num-common} specifies the fixed instance and the solvers
behind Figure~\ref{fig:ufm-common-instance}.
Appendices~\ref{sub:num-selection}--\ref{sub:num-correction} repeat each
check on random instances of several sizes, in the order of
Section~\ref{sub:theory}: selection by the proximal regularizer
(Lemma~\ref{sub:selection-main}), convergence to the BW barycenter
(Theorem~\ref{sub:oracle}), the three terms of the gap
(Proposition~\ref{sub:channels}) and the moment correction
(Proposition~\ref{sub:correction}).
All results of Section~\ref{sub:theory} hold for arbitrary positive weights $p_m$ summing to one,
with $\mu_g$, $\Sigma_{\mathrm{within}}$ and $\Sigma_{\cen}$ formed with these weights.
The instances in Appendices~\ref{sub:num-bw}--\ref{sub:num-correction} therefore draw the weights
at random instead of setting $p_m=N_m/N$.
Curves show medians over instances; in Figures~\ref{sub:fig-selection} and~\ref{sub:fig-gap-terms}, shaded bands show interquartile ranges.
Appendix~\ref{sub:supplementary-analysis} reports the numerical checks of the
finite-step analysis.

\subsection{The fixed instance of Figure~\ref{fig:ufm-common-instance}}
\label{sub:num-common}
Figure~\ref{fig:ufm-common-instance} illustrates the results of
Section~\ref{sub:theory} on one fixed instance with $C=P=2$, $M=3$,
$N_m=256$, $p_m=1/3$ and $\lambda_H=\lambda_W=0.1$.
Writing $R_\theta$ for planar rotation through $\theta$ degrees, the
client moments and the initial Gram matrix are
\begin{align*}
 \Sigma_1&=R_{-40}\diag(3,0.11)R_{-40}^\top,& \mu_1&=(0,-2.2)^\top,\\
 \Sigma_2&=R_{10}\diag(1.1,0.45)R_{10}^\top,& \mu_2&=(0,0)^\top,\\
 \Sigma_3&=R_{55}\diag(2.4,0.13)R_{55}^\top,& \mu_3&=(0,2.2)^\top,\\
 G_0&=R_{100}\diag(2.5,0.09)R_{100}^\top,& W_0&=G_{0}^{1/2}.
\end{align*}
Each target matrix is generated by centering and whitening a Gaussian draw
with random seed 0, then applying the prescribed mean and covariance.
The empirical moments equal these values up to rounding, and concatenating
the three equal-size matrices gives the same $\Sigma_{\cen}$ as
\eqref{sub:pooled-covariance}.
A Gram matrix $G$ is drawn as the ellipse
$\{G^{1/2}u:\norm u_2=1\}$, whose semiaxes are the square roots of its
eigenvalues.

\paragraph{Local optimization and selection.}
Panel (a) follows client 1 during joint training of $(W,b,H)$ with Adam,
a learning rate of $0.02$ and no proximal regularizer, starting from $W_0$,
$b=\mu_1$ and the ridge-optimal $H$ at $W_0$. Its snapshots show the head
Gram matrix after $E=0,20,80$ and $240$ full-batch steps, together with
client~1's optimal Gram matrix $G_1$, on the coordinate scale of panel (d).
The other panels use the profiled objective \eqref{sub:profile}, in which
$(b,H)$ are minimized out in closed form. For $\rho>0$, the clients minimize
the penalized profiled objective with SciPy's \texttt{trust-exact} method,
using the analytic gradient and Hessian. For a returned head $W_\rho$, we
measure the relative Gram error and the relative selection error
\[
 e_G(\rho)=\frac{\norm{W_\rho W_\rho^\top-G_m}_F}{\norm{G_m}_F},
 \qquad
 e_\Pi(\rho)=\frac{\norm{W_\rho-\Pi_m(W)}_F}{\norm{\Pi_m(W)}_F}.
\]
For panel (b), the broadcast is fixed at $W_0$, and $\rho$ decreases over
11 logarithmically spaced values from $10^{-1}$ to $10^{-6}$; the first
solve starts from $W_0$, and each later one from the previous solution.
Panel (b) shows the largest of each error over the three clients; at
$\rho=10^{-6}$, they are $4.66\times10^{-6}$ and $2.48\times10^{-6}$.

\paragraph{Convergence to the barycenter.}
For panels (c) and (d), each client starts its local solve at the selected
head $\Pi_m(W_t)$ and then minimizes the penalized objective numerically,
for 18 rounds at each $\rho\in\{10^{-1},10^{-2},10^{-3},10^{-4}\}$.
Panel (c) shows $\rho=10^{-1}$ and $10^{-3}$, together with exact selection
and with unregularized \mbox{L-BFGS} started from the broadcast head in each round.
After 18 rounds, the relative error of $G_t$ to $G_\star$ is
$3.49\times10^{-2}$ at $\rho=10^{-1}$, $2.55\times10^{-4}$ at
$\rho=10^{-3}$ and $3.11\times10^{-3}$ without the regularizer; with exact
selection, it falls below $10^{-14}$ by round 16.
Over all 303 profiled solves, the largest gradient norm is
$4.14\times10^{-9}$. Panel (d) shows the shared Gram matrix at
$\rho=10^{-3}$ at rounds 0, 1, 2 and 18, on one coordinate scale.
The reference $G_\star$ agrees to a relative error of $1.45\times10^{-15}$
with an independent minimization of the variational objective in
\eqref{sub:barycenter} over Cholesky factors.

\paragraph{Gap and correction.}
Panel (e) evaluates the three terms of Proposition~\ref{sub:channels} for
the original moments and for three modifications: equal means
$\mu_m=\mu_g$, equal covariances $\Sigma_m=\Sigma_{\mathrm{within}}$, and
both. All other quantities remain fixed. With the original moments, the
traces of $\mathcal M_\mu$, $\mathcal M_\Sigma$ and $\mathcal M_A$ are
$1.02$, $2.90\times10^{-1}$ and $1.93\times10^{-1}$.
Equalizing the means removes only $\mathcal M_\mu$, equalizing the
covariances removes $\mathcal M_\Sigma$ and $\mathcal M_A$, and equalizing
both removes the gap, up to rounding. The computed terms also satisfy
\eqref{sub:three-channel} and \eqref{sub:intrinsic-variance} and are
positive semidefinite, up to rounding.
Panel (f) applies the scaled correction $\gamma C_m$ under exact selection,
for $\gamma\in\{0,0.9,0.99,1\}$; client $m$ then has the optimal Gram matrix
$\phi((1-\gamma)\Sigma_m+\gamma\Sigma_{\cen})$. After one round, the
relative error of $G_t$ to $G_{\cen}$ is $8.08\times10^{-16}$ with the full
correction, and $3.53\times10^{-2}$ and $3.46\times10^{-3}$ for $\gamma=0.9$
and $0.99$. Without correction, it settles at the relative gap of this
instance, $5.78\times10^{-1}$.
\FloatBarrier

\subsection{Selection by a vanishing proximal regularizer}
\label{sub:num-selection}
We test Lemma~\ref{sub:selection-main}(ii) on 128 random fully active
instances for each $C\in\{2,4,8\}$ and $P\in\{C,2C\}$.
We set $\lambda_H=\lambda_W=1$ and $N_m=64$.
The client means are standard normal, and the covariances have random
orthogonal eigenbases with eigenvalues drawn uniformly from
$[2,5]$; the samples match these moments exactly.
At a fixed broadcast $W$ with singular values drawn uniformly from
$[0.7,1.5]$, each client minimizes
$\mathcal L_m+\frac{\rho}{2}\norm{W_m-W}_F^2$ jointly over $(W_m,b,H)$,
using SciPy's \texttt{trust-exact} solver with analytic derivatives.
Each weight is warm-started from the solution at the next larger weight.
Figure~\ref{sub:fig-selection} shows the medians and interquartile ranges of
the errors $e_G(\rho)$ and $e_\Pi(\rho)$ of Appendix~\ref{sub:num-common}.
For $\rho\le10^{-1}$, both errors decrease in proportion to $\rho$. At the
smallest weight shown, $\rho=1.78\times10^{-6}$, the medians of $e_G$ lie
between $5.15\times10^{-7}$ and $6.47\times10^{-7}$, and those of $e_\Pi$
between $2.70\times10^{-7}$ and $3.67\times10^{-7}$.
At $\rho=0$, the objective determines the Gram matrix but not the head.
The left slot of each panel shows this difference: exact optimal heads,
sampled on the optimal orbit, have Gram errors of at most
$3.30\times10^{-15}$, whereas unregularized \mbox{L-BFGS} returns from six
independent standard-normal head initializations per instance
all reach the optimal Gram matrix, with Gram errors below $10^{-3}$, but their
selection errors have a median of $1.42$.

\begin{figure}[!htbp]
\centering
\includegraphics[width=\linewidth]{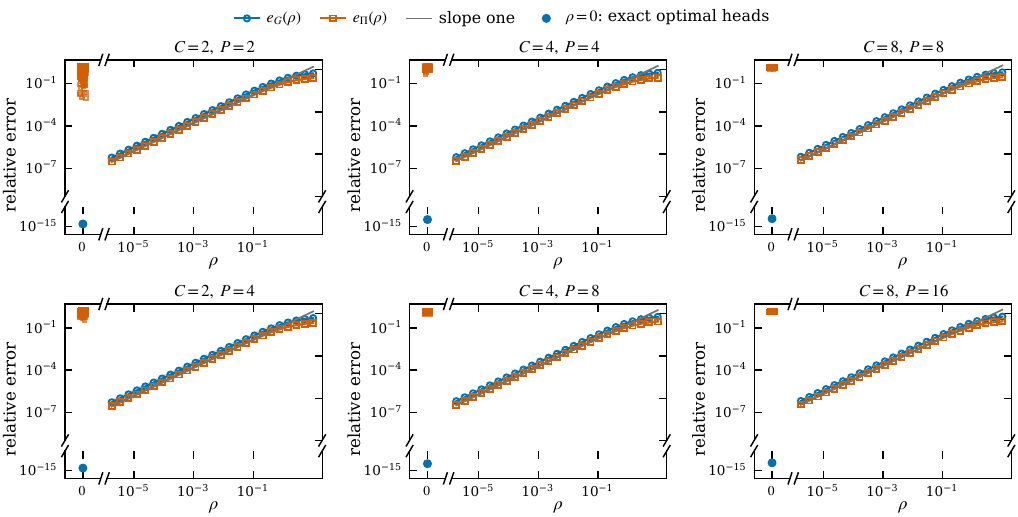}
\caption{\textbf{Selection by a vanishing proximal regularizer
(Lemma~\ref{sub:selection-main}).} Relative Gram error $e_G(\rho)$ (circles)
and selection error $e_\Pi(\rho)$ (squares) in six $(C,P)$ settings, with
medians and interquartile ranges over 128 instances. The gray line has slope
one. At $\rho=0$, the circle marks exact optimal heads and the squares mark
unregularized \mbox{L-BFGS} returns from random starts.}
\label{sub:fig-selection}
\end{figure}

In networks trained with a fixed local budget, agreement with $G_\star$ is
best at an intermediate weight (Appendix~\ref{app:proximal-weight}). At
$\rho=10^{-5}$, finite local training leaves the heads farther from the
selected heads, and at large $\rho$ the clients move away from their optima.
\FloatBarrier

\subsection{Convergence to the BW barycenter}
\label{sub:num-bw}
We test Theorem~\ref{sub:oracle} on 16 random fully active instances for each
$C\in\{2,4,8\}$ and $M\in\{3,8\}$, whose client covariances do not commute.
We set $P=2C$ and $\lambda_H=\lambda_W=1$.
Each $G_m$ has a random orthogonal eigenbasis and eigenvalues drawn uniformly
from $[0.4,4]$, and we set $\Sigma_m=(G_m+I_C)^2$.
The entries of the client means $\mu_m$ and the initial head $W_0$ are
independent standard normal.
Each client has $N_m=2C$ samples whose empirical mean and covariance equal the
prescribed $\mu_m$ and $\Sigma_m$, and the aggregation weights are drawn from a
symmetric Dirichlet distribution with concentration $2$ per client.
Figure~\ref{sub:fig-bw} shows the median relative error
$\norm{G_t-G_\star}_F/\norm{G_\star}_F$ of the shared Gram matrix over 18
rounds, with $G_\star$ obtained from a separate run of \eqref{sub:gram-map}
to convergence.

\begin{figure}[!htbp]
\centering
\includegraphics[width=\linewidth]{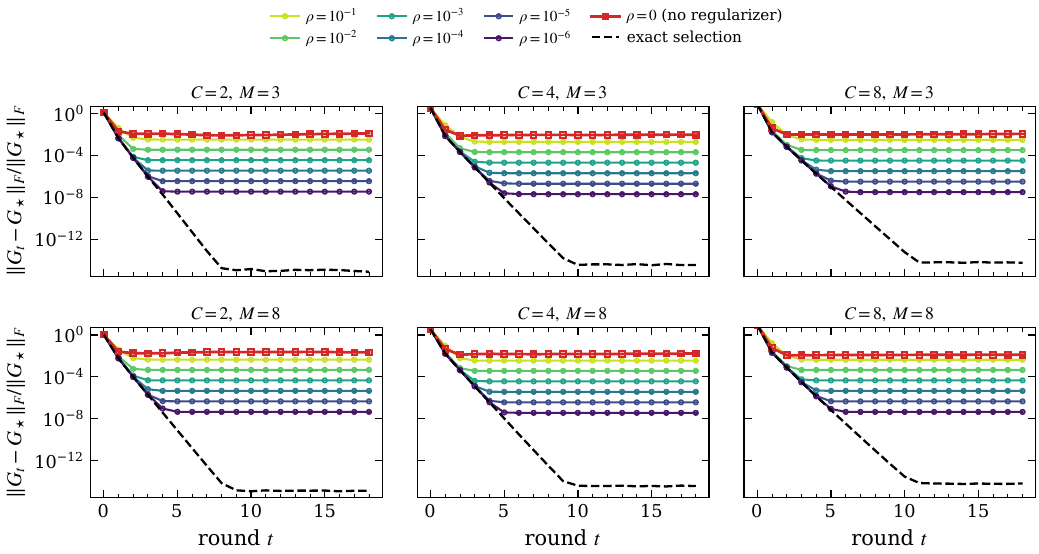}
\caption{\textbf{Convergence to the BW barycenter (Theorem~\ref{sub:oracle}).}
Median relative error of the shared Gram matrix to $G_\star$ over 16 instances
in each of six problem sizes. Dashed: exact selection. Colored circles:
proximal returns at $\rho=10^{-1},\ldots,10^{-6}$. Red squares: unregularized
joint training from the broadcast head.}
\label{sub:fig-bw}
\end{figure}

With exact selection, the error decreases geometrically to below $10^{-14}$.
For six weights $\rho\in\{10^{-1},\ldots,10^{-6}\}$, each client minimizes the
penalized objective jointly over $(W_m,b,H)$ with SciPy's
\texttt{trust-exact} solver and analytic derivatives, starting from
$\Pi_m(W_t)$ and its ridge-optimal features. The error then levels off at a
value roughly proportional to $\rho$. Unregularized joint training with Adam
from the broadcast head levels off between $9.84\times10^{-3}$ and
$2.13\times10^{-2}$.
Trained networks show the same ordering
(Table~\ref{tab:dnn-training-summary}): proximal training closely follows the
exact Gram iteration, while ordinary training ends near $G_\star$ with a
larger error.

\FloatBarrier

\subsection{The three terms of the gap to centralized training}
\label{sub:num-channels}
We test the three predictions of Proposition~\ref{sub:channels}: each term is
positive semidefinite, the gap vanishes exactly when the client means and
covariances agree, and $\Tr\mathcal M_A$ equals the BW variance of the
clients' optimal Gram matrices. By \eqref{sub:pairwise-bounds}, this variance
lies between $D$ and $2D$, where
$D=\sum_{m<n}p_mp_nd_{\mathrm{BW}}^2(G_m,G_n)$ is the pairwise dispersion.
Figure~\ref{sub:fig-gap-terms} reports these checks.
All panels use $\lambda_H=\lambda_W=0.1$ and standard-normal
base client means.

\begin{figure}[!htbp]
\centering
\includegraphics[width=\linewidth]{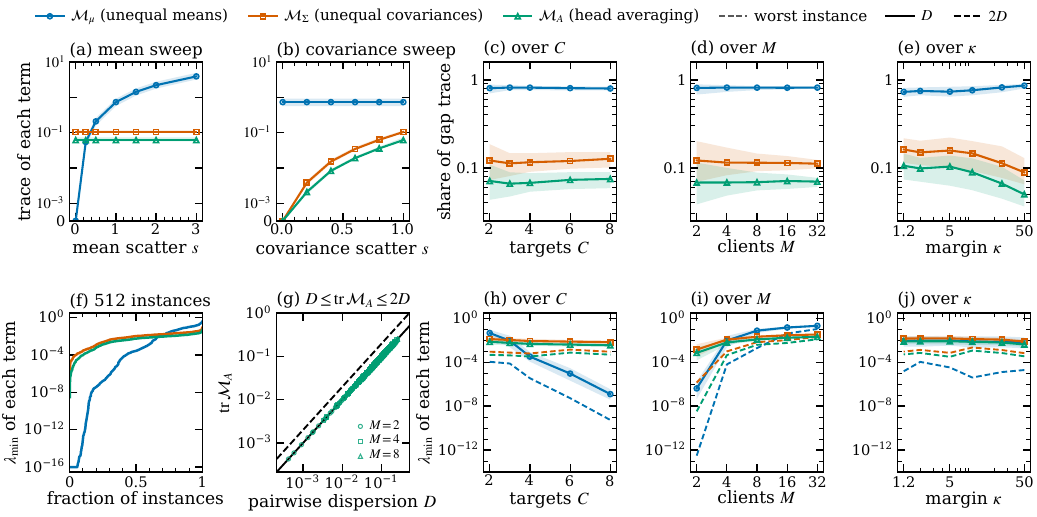}
\caption{\textbf{The three terms of the gap $G_{\cen}-G_\star$
(Proposition~\ref{sub:channels}).} Blue, orange and green identify
$\mathcal M_\mu$, $\mathcal M_\Sigma$ and $\mathcal M_A$ throughout; circles,
squares and triangles mark the terms in (a--e) and (h--j), while in (g) they
distinguish $M=2,4,8$. Curves are medians and bands are interquartile ranges.
(a, b) Trace of each term when the client means or the client covariances
are varied with strength $s$ while the other is held fixed. The vertical axis is linear
below $10^{-3}$ so that exact zeros are visible.
(c--e) Share of the gap trace carried by each term against the number of
targets $C$, the number of clients $M$ and a lower bound $\kappa$ on the fully active margin.
(f) Smallest eigenvalue of each term over 512 instances, sorted; numerical
zeros are drawn at $10^{-16}$.
(g) $\Tr\mathcal M_A$ against the pairwise dispersion $D$, with the bounds $D$
and $2D$ of \eqref{sub:pairwise-bounds}.
(h--j) Smallest eigenvalue of each term over the sweeps of (c--e); dashed
curves show the worst of the 256 instances in each setting.}
\label{sub:fig-gap-terms}
\end{figure}

\paragraph{Which heterogeneity produces which term.}
Panels (a) and (b) use two controlled sweeps with $C=3$ and $M=4$, each with
the eight seeds $\{55,56,57,58,59,60,61,62\}$ at every strength $s$.
For each seed, the base covariances have random orthogonal eigenbases and
eigenvalues uniform in $[0.3,3]$.
The weights are normalized independent uniform draws from
$[0.5,1.5]$.
The mean
sweep scales the separation of the client means by
$s\in\{0,0.25,0.5,1,1.5,2,3\}$ and keeps the client covariances fixed. The term $\mathcal M_\mu$ then grows
from exactly zero at $s=0$ to a median trace of $3.90$ at $s=3$, while
$\mathcal M_\Sigma$ and $\mathcal M_A$ keep their median traces of
$1.04\times10^{-1}$ and $6.12\times10^{-2}$. The covariance sweep moves each client
covariance along
$\Sigma_m(s)=\Sigma_{\mathrm{within}}+s(\Sigma_m-\Sigma_{\mathrm{within}})$
for $s\in\{0,0.2,0.4,0.6,0.8,1\}$, which keeps the client means and the weighted within-client covariance fixed.
Now $\mathcal M_\Sigma$ and $\mathcal M_A$ grow from zero, with traces below
$2\times10^{-14}$ at $s=0$, to the same two values, and $\mathcal M_\mu$
keeps its median trace of $7.30\times10^{-1}$. Equalizing the moments of the 512 instances
described below gives the same pattern: equal means remove $\mathcal M_\mu$,
equal covariances remove $\mathcal M_\Sigma$ and $\mathcal M_A$, and
equalizing both removes the gap. Every removed term has trace below
$2\times10^{-14}$.

\paragraph{Positive semidefiniteness.}
Panel (f) sorts the smallest eigenvalue of each term over 512 random
fully active instances, with $M\in\{2,4,8\}$ clients, $C=P\in\{2,3,5,8\}$,
randomly rotated covariances with eigenvalues drawn uniformly from $[0.25,3]$, and symmetric
Dirichlet weights with concentration $2$ per client.
No eigenvalue is negative beyond rounding: the minimum over all 1536
values is $-2.42\times10^{-15}$. Sixty values of $\mathcal M_\mu$ lie below
$10^{-12}$, and 29 of them are drawn at the floor of the panel. They belong
to two-client instances with $C\ge5$, where the mean scatter $\Sigma_\mu$
has rank one, so the trailing eigenvalues of
$\phi(\Sigma_{\mathrm{within}}+\Sigma_\mu)-\phi(\Sigma_{\mathrm{within}})$
are small but, for generic client means, not exactly zero.
Panels (h)--(j) repeat the test on 256 new instances per setting while one
parameter varies: the number of targets $C$ with $M=4$, the number of clients
$M$ with $C=3$, and the lower bound $\kappa$ on the fully active margin with $C=3$ and $M=4$.
These sweeps use the weight distribution of panel (f); the $C$ and $M$
sweeps also use its covariance distribution.
In the last sweep, the covariance eigenvalues are drawn uniformly from
$[\kappa\lambda_H\lambda_W,3]$, and $\kappa$ decreases from $50$ to $1.2$,
toward the boundary of the fully active regime. The smallest eigenvalue of the
worst instance stays positive in every setting; its minimum,
$3.07\times10^{-13}$, occurs for $\mathcal M_\mu$ with two clients.

\paragraph{The BW variance.}
Panel (g) plots $\Tr\mathcal M_A$ against $D$ for the 512 instances. Every
point lies between the lines $D$ and $2D$. The ratio $\Tr\mathcal M_A/D$
ranges over $[1.000,1.006]$ and equals one for two clients. In that case the
barycenter lies on the BW geodesic between $G_1$ and $G_2$, at distance
$p_2d_{\mathrm{BW}}(G_1,G_2)$ from $G_1$~\citep{doi:10.1137/100805741}, so
$\Tr\mathcal M_A=p_1p_2d_{\mathrm{BW}}^2(G_1,G_2)=D$ exactly. The identity
\eqref{sub:intrinsic-variance} holds to a relative residual of at most
$2.38\times10^{-12}$, and the weighted scatter of the selected heads,
$\sum_mp_m(U_m^\star-W_\star)(U_m^\star-W_\star)^\top$, matches
$\mathcal M_A=\sum_mp_mG_m-G_\star$ to a relative residual of at most
$3.24\times10^{-12}$.

\paragraph{Composition of the gap.}
Panels (c)--(e) show the share of the gap trace carried by each term on the
instances of panels (h)--(j), whose client means are standard normal. The term
$\mathcal M_\mu$ carries 73--86\% of the gap trace, $\mathcal M_\Sigma$
carries 9--16\% and $\mathcal M_A$ carries 5--11\% (medians). The relative gap
$\norm{G_{\cen}-G_\star}_F/\norm{G_{\cen}}_F$ has medians between $1.68\times10^{-1}$
and $3.04\times10^{-1}$. In networks, removing the mean term by clientwise centering also
leaves a positive gap (Appendix~\ref{app:moment-dnn}), consistent with the
remaining terms $\mathcal M_\Sigma+\mathcal M_A$.
\FloatBarrier

\subsection{One-round recovery by the moment correction}
\label{sub:num-correction}
We test Proposition~\ref{sub:correction} under exact selection on the six
problem sizes of Figure~\ref{sub:fig-bw}, with 32 random fully active
instances per size whose client covariances do not commute.
Proposition~\ref{sub:correction} predicts that the corrected clients return
identical heads, so that $G_t$ reaches $G_{\cen}$ after one round, while the
local biases $\mu_m$ average to $\mu_g$. The eigenvalues of each $G_m$ are
uniform in $[0.4,4]$, the client means are standard normal, the weights are
Dirichlet, $P=2C$ and $\lambda_H=\lambda_W=1$. Every client returns the
selected head of its corrected objective, so each round is a closed-form map
and involves no solver.

\begin{figure}[!htbp]
\centering
\includegraphics[width=\linewidth]{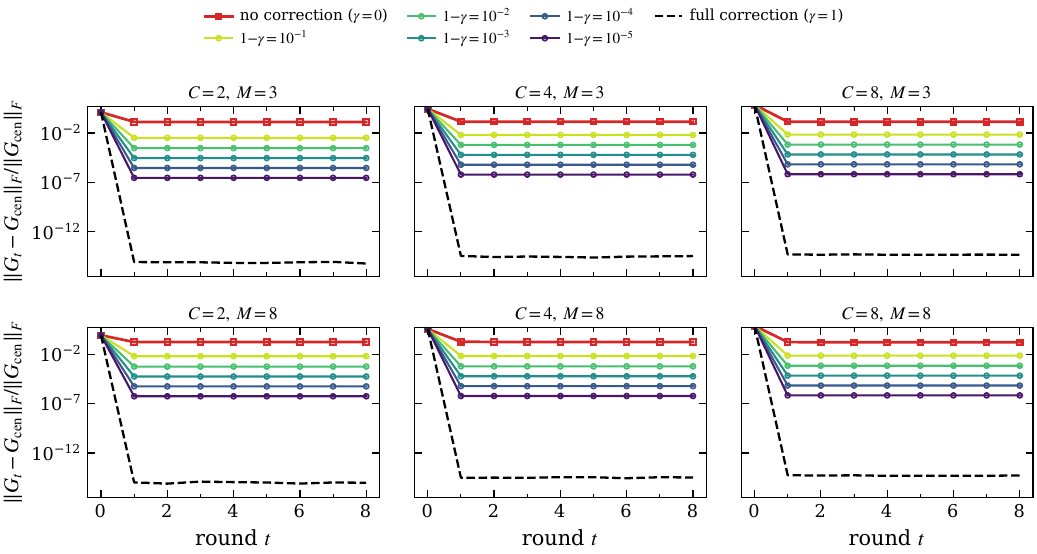}
\caption{\textbf{One-round recovery of the centralized Gram matrix
(Proposition~\ref{sub:correction}).} Median relative error of the shared Gram
matrix to $G_{\cen}$ over 32 instances in each of six problem sizes, under
exact selection. Red squares: no correction; the plateau is the gap of
Proposition~\ref{sub:channels}. Dashed: the correction $C_m$ of
\eqref{sub:correction-term}. Colored: the scaled correction $\gamma C_m$, with
deficit $1-\gamma$ from $10^{-1}$ to $10^{-5}$.}
\label{sub:fig-correction}
\end{figure}

Without the correction, $G_t$ converges to $G_\star$, and its relative
distance to $G_{\cen}$ settles at the relative gap of the instance, with
medians between $1.29\times10^{-1}$ and $1.75\times10^{-1}$. With the correction $C_m$, the relative
distance after one round is at most $7.92\times10^{-15}$ over all 192
instances. The colored curves scale the correction to $\gamma C_m$. Because
$F(W;T)$ is affine in $T$, client $m$ then minimizes
$F(W;(1-\gamma)\Sigma_m+\gamma\Sigma_{\cen})$, and the remaining distance is
proportional to the deficit $1-\gamma$ from $10^{-1}$ to $10^{-5}$.

A second check solves the corrected profiled objective numerically, with
\mbox{L-BFGS} and a positive proximal weight, in a fixed instance with $M=4$ and
$C=P=3$. At the final weight $\rho=10^{-5}$, one corrected round brings $G_t$
within a relative error of $9.13\times10^{-6}$ of $G_{\cen}$, and the four
returned heads lie within a Frobenius distance of $3.37\times10^{-7}$ of the
first. With half of the correction, the returned heads differ by $1.66\times10^{-1}$,
and $G_t$ misses $G_{\cen}$ by a relative error of $1.03\times10^{-1}$. Across 512 random
instances, the identity $F(W;\Sigma_m)+C_m(W)=F(W;\Sigma_{\cen})$ holds with
residuals of at most $3.30\times10^{-16}$ in value and $9.51\times10^{-15}$ in
gradient.
In trained networks, corrected proximal training approaches $G_{\cen}$ within
the first few rounds (Appendix~\ref{app:correction-rounds}).
\FloatBarrier

\section{Details of the DNN experimental setup}
\label{app:exp-details}
This appendix describes the datasets, models, federated training and evaluation used in Section~\ref{sec:experiments} and Appendix~\ref{app:supp-dnn}. Appendix~\ref{app:exp-overview} defines the experimental settings and summarizes the main and supplementary experiments.

\subsection{Datasets}
\label{app:dnn-datasets}
The DNN experiments use five tabular and five image datasets; Table~\ref{tab:datasets} lists their input and target dimensions, numbers of clients and sample counts. For Beijing Multi-Site Air Quality~\citep{beijing_multi-site_air_quality_501,10.1098/rspa.2017.0457}, we predict PM$_{2.5}$ and NO$_2$ concentrations from PM$_{10}$ concentration, temperature, pressure, dew-point temperature, precipitation, wind speed, month and hour. The other tabular datasets are four MuJoCo environments (Swimmer, Hopper, HalfCheetah and Walker; \citealp{6386109}), in which, following \citet{NEURIPS2024_e4748b6b}, actions are predicted from observations. For each environment we use the first 4,000 training transitions from JAT~\citep{gallouedec2024jacktradesmastersome}.

The image datasets differ in how directly they pose multivariate regression. MPIIGaze~\citep{Zhang_2015_CVPR} and 300W-LP~\citep{Zhu_2016_CVPR} are multivariate regression datasets by design, for gaze and head-pose estimation: we predict gaze pitch and yaw from eye images and head pitch, yaw and roll from face images. dSprites~\citep{dsprites17} and 3D Shapes~\citep{3dshapes18} are public datasets labeled with their generative factors. Regressing these factors from images has been studied before~\citep{schott2022visual}, but without a standard protocol, so we choose the regressed factors and the data split: horizontal position, vertical position and scale for dSprites, and object hue, scale and orientation for 3D Shapes. We construct CIFAR-geo as a regression task: we randomly shift and zoom each CIFAR-10~\citep{Krizhevsky09learningmultiple} image and predict the horizontal translation, vertical translation and log scale. MuJoCo observations and action targets retain their stored scale in both training and evaluation. The targets of the other datasets are centered and scaled with pooled training statistics, in training and at evaluation.

For the main experiments, we divide each dataset among clients using two partition schemes. For CIFAR-geo, we design the client groups to have different distributions of shift and zoom. For the other nine datasets, we project each training target onto the leading eigenvector of the pooled training-target second-moment matrix, sort the training samples by this projection and divide them into equally sized groups, one per client. Both schemes give the clients different target means and covariances.

The sample counts in Table~\ref{tab:datasets} are totals over all clients. In CIFAR-geo, each client holds about 3,000 training samples, and the test samples are split evenly at random, 750 per client. In the other nine datasets, every client holds the same number of training samples. In the other four image datasets, test samples are assigned by their target projections to the nearest client training interval, giving test sets of different sizes. This target-based assignment provides an oracle evaluation protocol.

All datasets are used under their published terms. Beijing Multi-Site Air Quality is released under CC BY 4.0. Apache 2.0 covers dSprites, 3D Shapes and the JAT dataset, from which we take the MuJoCo transitions. MPIIGaze is released under CC BY-NC-SA 4.0 for non-commercial scientific use. 300W-LP is synthesized from 300-W, whose data are provided for research purposes only. The official CIFAR-10 page states no license and asks users to cite \citet{Krizhevsky09learningmultiple}.

{\small
\setlength{\tabcolsep}{4pt}
\renewcommand{\arraystretch}{1.10}
\begin{longtable}{@{}lccccc@{}}
\caption{\textbf{Datasets and sample counts} for the main DNN experiments. Test error is reported only for the image datasets (Appendix~\ref{app:prediction}).}\label{tab:datasets}\\
\toprule
Dataset & Input dimension & Target dimension & Clients & Training samples & Test samples \\
\midrule
\endfirsthead
\multicolumn{6}{l}{\textit{Table \thetable\ continued}}\\
\toprule
Dataset & Input dimension & Target dimension & Clients & Training samples & Test samples \\
\midrule
\endhead
\midrule
\multicolumn{6}{r}{\textit{Continued on next page}}\\
\endfoot
\bottomrule
\endlastfoot
\multicolumn{6}{@{}l}{\textit{Tabular datasets}} \\*
Beijing & 8 & 2 & 4 & 9,000 & \textemdash \\
Swimmer & 8 & 2 & 4 & 4,000 & \textemdash \\
Hopper & 11 & 3 & 4 & 4,000 & \textemdash \\
HalfCheetah & 17 & 6 & 4 & 4,000 & \textemdash \\
Walker & 17 & 6 & 4 & 4,000 & \textemdash \\
\midrule
\multicolumn{6}{@{}l}{\textit{Image datasets}} \\*
MPIIGaze & $36\times60\times1$ & 2 & 8 & 12,000 & 3,000 \\
300W-LP & $64\times64\times3$ & 3 & 8 & 12,000 & 3,000 \\
dSprites & $64\times64\times1$ & 3 & 8 & 12,000 & 3,000 \\
3D Shapes & $64\times64\times3$ & 3 & 8 & 12,000 & 3,000 \\
CIFAR-geo & $40\times40\times3$ & 3 & 4 & 12,000 & 3,000 \\
\end{longtable}
}

\subsection{Models}
The main experiments use a residual MLP (ResMLP) for the tabular datasets and ResNet-18~\citep{He_2016_CVPR} for the image datasets. The ResMLP follows the architecture of \citet{NEURIPS2025_93da65de}, with three residual blocks of width 1024 and PReLU activations~\citep{He_2015_ICCV}. All networks are trained from scratch and output a 512-dimensional feature through a linear layer with no activation.

Architecture comparisons also use ResNet-34~\citep{He_2016_CVPR} and a residual CNN implemented as a ResNet-18 variant. The residual CNN has four stages of two basic blocks with 64, 128, 256 and 512 channels. It replaces ResNet-18's BatchNorm with eight-group GroupNorm~\citep{Wu_2018_ECCV} and its $7\times7$ stride-2 stem and max pooling with a $3\times3$ stride-1 stem. The first block of each of the last three stages downsamples with stride 2. For ResNet-18 and ResNet-34, the classification layer is replaced by a 512-to-512 linear feature projection after global average pooling. The residual CNN uses the same pooling and projection.

\subsection{Federated training}
\label{app:return-protocol}
Unless stated otherwise, all DNN experiments share the round protocol, local objective and optimizer below; the four training procedures differ only in the terms added to the objective. In each round, all clients receive the shared head $(W_t,b_t)$, train it together with their private backbone, and upload the resulting head. The server averages these heads with weights $p_m=N_m/N$. Each backbone and any BatchNorm running statistics remain private throughout training. The main experiments start clients from the same randomly initialized backbone and use eight communication rounds and $E=4000$ local epochs per round. Appendix~\ref{app:esweep-dnn} varies $E$, and Appendix~\ref{app:dnn-scope} also studies full-model averaging.

Training uses a minibatch version of \eqref{sub:loss}, with $H_m$ given by the network features, $\lambda_H=\lambda_W=0.01$ and zero optimizer weight decay. The head bias is unregularized. The fully active condition $\lambda_{\min}(\Sigma_m)>\lambda_H\lambda_W$ holds for every client in the reported experiments: the smallest client covariance eigenvalue is $9.11\times10^{-3}$, compared with $\lambda_H\lambda_W=10^{-4}$.

The main experiments compare four training procedures. Ordinary training optimizes this objective. Proximal training adds $\rho\|W-W_t\|_F^2/2$ with $\rho=10^{-3}$ to approximate the selection rule of Lemma~\ref{sub:selection-main}. Corrected ordinary training adds the head-only moment correction $C_m(W)$ in \eqref{sub:correction-term}, computed from client and pooled training-target moments. Corrected proximal training adds both terms.

The optimizer is AdamW~\citep{loshchilov2018decoupled}, reinitialized for each client's local training in every round. Its learning rate follows a cosine schedule from $10^{-3}$ to $10^{-5}$ within each round. We use shuffled minibatches of size 256 and clip the global gradient norm at 5.

\subsection{Evaluation}
\label{app:local-screen}
We evaluate each run by the geometry of the shared head, the local fit and prediction error of the trained models, and display the results as curves over rounds and as ellipses.

\paragraph{Head geometry and dynamics.}
The theoretical targets $G_\star$ and $G_{\cen}$ are computed from training-target moments, client weights and regularization. The procedures without correction are compared with $G_\star$, and the corrected procedures with $G_{\cen}$. For Gram matrices $G$ and $G'$, we report relative matrix error and direction error:
\[
 e_F(G,G')=\frac{\|G-G'\|_F}{\|G'\|_F},\qquad
 d(G,G')=\left\|\frac{G}{\|G\|_F}-\frac{G'}{\|G'\|_F}\right\|_F.
\]
Endpoint errors compare the final shared Gram $G_8$ with these fixed targets. For training without correction, trajectory error compares $G_t$ with the exact recursion \eqref{sub:gram-map} started from the same $G_0$, taking the median $e_F$ over rounds 1--8 for each run. Table~\ref{tab:dnn-training-summary} reports medians of these run-level quantities. For these runs, the weighted distance from the uploaded heads to the selected heads is
\[
 \delta_t=\sum_m p_m\|\widetilde W_{m,t}-\Pi_m(W_t)\|_F,
 \qquad \delta_{\mathrm{last}}=\delta_7.
\]

\paragraph{Local fitting and prediction.}
For training without correction, we evaluate the original objective \eqref{sub:loss} on each trained local model in evaluation mode before server averaging and report the relative gap $(\mathcal L_m-\mathcal L_m^\star)/\mathcal L_m^\star$. For the eigenvalues $s_i$ of $\Sigma_m$, the regression-UFM optimum~\citep{NEURIPS2024_e4748b6b} gives
\[
 \mathcal L_m^\star
 =\sum_{i:s_i>\lambda_H\lambda_W}
       \left(\sqrt{\lambda_H\lambda_W s_i}-\frac{\lambda_H\lambda_W}{2}\right)
  +\frac12\sum_{i:s_i\leq\lambda_H\lambda_W}s_i.
\]
Prediction uses the final shared head with each client's own backbone. MSE averages squared errors over samples and output coordinates, weighting each client by the number of samples in its evaluated subset; it excludes all penalties and correction terms.

\paragraph{Curves and geometric displays.}
The round curves show means of per-seed errors with one sample standard deviation, using the same seeds across compared procedures. Geometric displays use seed 0 and draw the boundary of $\mathcal E(G)=\{G^{1/2}u:\|u\|_2\leq1\}$, with common coordinates and equal axis units. HalfCheetah and Walker use the first three target coordinates of their six-dimensional Grams; all reported errors use the full matrices.

\subsection{Experiment overview}
\label{app:exp-overview}
We call a dataset paired with a network a setting. The main experiments use ten settings, pairing each tabular dataset with ResMLP and each image dataset with ResNet-18. Every setting is trained with the four procedures of Appendix~\ref{app:return-protocol} and five seeds $\{0,1,2,3,4\}$; Table~\ref{tab:dnn-cohorts} shows this coverage.

The supplementary comparisons in Table~\ref{tab:dnn-supp-overview} use three seeds $\{0,1,2\}$. Mean centering and the local-epoch and proximal-weight sweeps use the same ten settings. The architecture comparisons add the residual CNN and ResNet-34 on the five image datasets; the data-size comparison enlarges the training sets of CIFAR-geo and 3D Shapes; and the full-model comparison averages the backbones as well as the heads on the five tabular datasets. The prediction comparison uses the five image settings. Comparisons between training procedures pair runs by setting, seed, feature and head penalties, and local-epoch budget.

The alignment experiments in Appendix~\ref{app:alignment}, which Table~\ref{tab:dnn-supp-overview} does not list, impose the selection rule by aligning each trained head to the broadcast head instead of adding the proximal regularizer. They compare aligned, proximal and ordinary training on six settings (Beijing, Swimmer, Hopper, HalfCheetah, 3D Shapes and CIFAR-geo), with five seeds per setting and three seeds for the moment-corrected comparison, and measure the effect of alignment on a single aggregation in 60 runs.

\begin{table}[!htbp]
\centering\small
\setlength{\tabcolsep}{4pt}
\renewcommand{\arraystretch}{1.03}
\caption{\textbf{Main DNN experiments.} Coverage by setting and procedure in Appendix~\ref{app:add-results}.}
\label{tab:dnn-cohorts}
\label{tab:dnn-main-overview}
\begin{tabular*}{\linewidth}{@{\extracolsep{\fill}}lcccc@{}}
\toprule
\multicolumn{5}{c}{$\checkmark$: experiment repeated with five random seeds $\{0,1,2,3,4\}$} \\
\midrule
 & \multicolumn{2}{c}{Without correction} & \multicolumn{2}{c}{With moment correction} \\
\cmidrule(lr){2-3}\cmidrule(l){4-5}
Setting & Ordinary & Proximal & Ordinary & Proximal \\
\midrule
\multicolumn{5}{@{}l}{\textit{Tabular datasets}} \\
Beijing & $\checkmark$ & $\checkmark$ & $\checkmark$ & $\checkmark$ \\
Swimmer & $\checkmark$ & $\checkmark$ & $\checkmark$ & $\checkmark$ \\
Hopper & $\checkmark$ & $\checkmark$ & $\checkmark$ & $\checkmark$ \\
HalfCheetah & $\checkmark$ & $\checkmark$ & $\checkmark$ & $\checkmark$ \\
Walker & $\checkmark$ & $\checkmark$ & $\checkmark$ & $\checkmark$ \\
\midrule
\multicolumn{5}{@{}l}{\textit{Image datasets}} \\
MPIIGaze & $\checkmark$ & $\checkmark$ & $\checkmark$ & $\checkmark$ \\
300W-LP & $\checkmark$ & $\checkmark$ & $\checkmark$ & $\checkmark$ \\
dSprites & $\checkmark$ & $\checkmark$ & $\checkmark$ & $\checkmark$ \\
3D Shapes & $\checkmark$ & $\checkmark$ & $\checkmark$ & $\checkmark$ \\
CIFAR-geo & $\checkmark$ & $\checkmark$ & $\checkmark$ & $\checkmark$ \\
\bottomrule
\end{tabular*}
\par\smallskip
\begin{minipage}{\linewidth}\footnotesize
The first five settings use ResMLP; the five image settings use ResNet-18. Each procedure uses eight rounds and $E=4000$. Proximal training adds $\rho\|W-W_t\|_F^2/2$ with $\rho=10^{-3}$. Moment correction adds $C_m$ from Proposition~\ref{sub:correction}.
\end{minipage}
\end{table}

\begin{table}[!htbp]
\centering\footnotesize
\setlength{\tabcolsep}{2pt}
\renewcommand{\arraystretch}{1.04}
\caption{\textbf{Supplementary DNN experiments.} Coverage by setting and condition; sweep columns give coverage at every parameter value.}
\label{tab:dnn-supp-overview}
\begin{tabular*}{\linewidth}{@{\extracolsep{\fill}}l*{11}{c}@{}}
\toprule
\multicolumn{12}{c}{$\checkmark$: experiment repeated with three random seeds $\{0,1,2\}$} \\
\midrule
 & \ref{app:moment-dnn} & \ref{app:esweep-dnn} & \ref{app:proximal-weight} & \multicolumn{7}{c}{\ref{app:dnn-scope}} & \ref{app:prediction} \\
\cmidrule(lr){5-11}
Setting & Means & $E$ & $\rho$ & CNN & \shortstack{CNN\\$+C_m$} & RN34 & Data & \shortstack{Full\\$E=4000$} & \shortstack{Full\\$E=400$} & \shortstack{Full\\$+C_m$} & Predict \\
\midrule
\multicolumn{12}{@{}l}{\textit{Tabular datasets}} \\
Beijing & $\checkmark$ & $\checkmark$ & $\checkmark$ & \textemdash & \textemdash & \textemdash & \textemdash & $\checkmark$ & $\checkmark$ & $\checkmark$ & \textemdash \\
Swimmer & $\checkmark$ & $\checkmark$ & $\checkmark$ & \textemdash & \textemdash & \textemdash & \textemdash & $\checkmark$ & $\checkmark$ & $\checkmark$ & \textemdash \\
Hopper & $\checkmark$ & $\checkmark$ & $\checkmark$ & \textemdash & \textemdash & \textemdash & \textemdash & $\checkmark$ & $\checkmark$ & $\checkmark$ & \textemdash \\
HalfCheetah & $\checkmark$ & $\checkmark$ & $\checkmark$ & \textemdash & \textemdash & \textemdash & \textemdash & $\checkmark$ & $\checkmark$ & $\checkmark$ & \textemdash \\
Walker & $\checkmark$ & $\checkmark$ & $\checkmark$ & \textemdash & \textemdash & \textemdash & \textemdash & $\checkmark$ & $\checkmark$ & $\checkmark$ & \textemdash \\
\midrule
\multicolumn{12}{@{}l}{\textit{Image datasets}} \\
MPIIGaze & $\checkmark$ & $\checkmark$ & $\checkmark$ & $\checkmark$ & $\checkmark$ & $\checkmark$ & \textemdash & \textemdash & \textemdash & \textemdash & $\checkmark$ \\
300W-LP & $\checkmark$ & $\checkmark$ & $\checkmark$ & $\checkmark$ & $\checkmark$ & $\checkmark$ & \textemdash & \textemdash & \textemdash & \textemdash & $\checkmark$ \\
dSprites & $\checkmark$ & $\checkmark$ & $\checkmark$ & $\checkmark$ & $\checkmark$ & $\checkmark$ & \textemdash & \textemdash & \textemdash & \textemdash & $\checkmark$ \\
3D Shapes & $\checkmark$ & $\checkmark$ & $\checkmark$ & $\checkmark$ & $\checkmark$ & $\checkmark$ & $\checkmark$ & \textemdash & \textemdash & \textemdash & $\checkmark$ \\
CIFAR-geo & $\checkmark$ & $\checkmark$ & $\checkmark$ & $\checkmark$ & $\checkmark$ & $\checkmark$ & $\checkmark$ & \textemdash & \textemdash & \textemdash & $\checkmark$ \\
\bottomrule
\end{tabular*}
\par\smallskip
\begin{minipage}{\linewidth}\footnotesize
An em dash marks a setting outside that comparison. Means: clientwise target centering. $E$: seven local-epoch budgets. $\rho$: seven proximal weights. CNN and RN34: residual CNN and ResNet-34. CNN + $C_m$: residual CNN with moment correction at $E=4000$. Data: larger training sets with $E=1000$. Full: backbone and head averaging at the stated budget. Full + $C_m$: full-model averaging with moment correction at $E=4000$. Predict: training and test MSE for ordinary, proximal and corrected proximal training.
\end{minipage}
\end{table}

\paragraph{Compute.}
Computation used NVIDIA GH200 GPUs (GH200 120GB model; 97,871 MiB of visible device memory). Each run used one GPU, with up to eight independent runs sharing a GPU. We estimate 1,415 occupied GPU-hours for the 764 successful runs underlying the reported DNN results by merging their execution intervals on each GPU, including intervals shared with other runs. Experiment configurations, manifests, results and figure-source hashes are retained in the experimental records.

\section{Supplementary DNN experimental results}
\label{app:supp-dnn}
This appendix reports the DNN comparisons behind Section~\ref{sec:experiments} and their extensions. Appendix~\ref{app:add-results} compares ordinary and proximal training with and without moment correction. Appendices~\ref{app:moment-dnn}--\ref{app:proximal-weight} test the gap decomposition and the dependence on the local budget and the proximal weight, and Appendix~\ref{app:alignment} imposes the selection rule by alignment instead of the proximal regularizer. Appendix~\ref{app:dnn-scope} extends the comparison to other architectures, larger training sets and full-model averaging, and Appendix~\ref{app:prediction} examines test error, which lies outside the UFM analysis. Appendix~\ref{app:exp-overview} lists the settings and seeds of each comparison.

\subsection{Selection and moment correction}
\label{app:add-results}
\label{app:proximal-results}
\label{app:anchor-dnn}
\label{app:correction-rounds}
\label{app:dnn-endpoints}
We first compare ordinary and proximal training across all main experiments, measuring agreement with $G_\star$ without correction and with $G_{\cen}$ under moment correction. Table~\ref{tab:dnn-training-summary} summarizes these comparisons. Without correction, the proximal regularizer improves endpoint agreement with $G_\star$ in every pair. It also reduces the median trajectory error relative to the exact iteration and the median distance to the selected heads.

\begin{table}[!htbp]
 \centering\small
 \setlength{\tabcolsep}{3pt}
 \caption{\textbf{Selection and moment correction.} Medians over the same 50 dataset--seed combinations. Without correction, endpoint errors compare with $G_\star$ and trajectory errors with the exact iteration from each run's $G_0$. All corrected errors compare with $G_{\cen}$. Bold marks the smaller value in each column.}
 \label{tab:dnn-training-summary}
 \label{tab:dnn-correction-first}
 \begin{tabular}{@{}lcccccc@{}}
 \toprule
 & \multicolumn{3}{c}{Without correction} & \multicolumn{3}{c}{With moment correction} \\
 \cmidrule(lr){2-4}\cmidrule(l){5-7}
 Training & Endpoint $e_F$ & Trajectory $e_F$ & $\delta_{\rm last}$ & Round 1 $d$ & Round 1 $e_F$ & Round 8 $e_F$ \\
 \midrule
 Ordinary & $1.63\times10^{-2}$ & $3.44\times10^{-2}$ & $1.85\times10^{-1}$ & $8.35\times10^{-2}$ & $3.08\times10^{-1}$ & $5.06\times10^{-3}$ \\
 Proximal & {\boldmath$2.90\times10^{-4}$} & {\boldmath$3.39\times10^{-4}$} & {\boldmath$4.47\times10^{-3}$} & {\boldmath$4.21\times10^{-3}$} & {\boldmath$2.29\times10^{-2}$} & {\boldmath$7.08\times10^{-5}$} \\
 \bottomrule
 \end{tabular}
\end{table}

With moment correction, proximal training reaches a smaller residual to $G_{\cen}$ in the first round and improves further over subsequent rounds. Figures~\ref{fig:dnn-selection} and~\ref{fig:dnn-selection-appendix} show these dynamics for all ten settings under proximal training. Appendix~\ref{sub:num-correction} reports the corresponding UFM check under exact selection.

Figure~\ref{fig:dnn-ordinary} gives the ordinary-training counterpart to Figures~\ref{fig:dnn-selection} and~\ref{fig:dnn-selection-appendix}, using the same seeds, axis limits and geometric scales. Figure~\ref{fig:dnn-protocol-comparison} directly compares ordinary and proximal training against $G_\star$ without correction and against $G_{\cen}$ with correction.

\begin{figure}[!htbp]
 \centering
 \includegraphics[width=\linewidth]{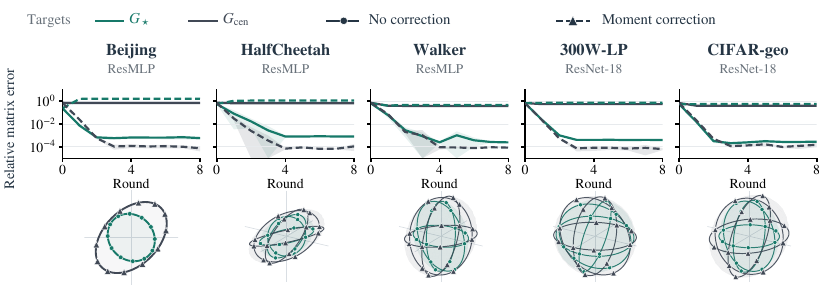}
 \caption{\textbf{Proximal training in the other five settings.} The design matches Figure~\ref{fig:dnn-selection}; six-dimensional settings use coordinates 1--3. All runs use $\rho=10^{-3}$.}
 \label{fig:dnn-selection-appendix}
\end{figure}

\begin{figure}[!htbp]
 \centering
 \includegraphics[width=\linewidth]{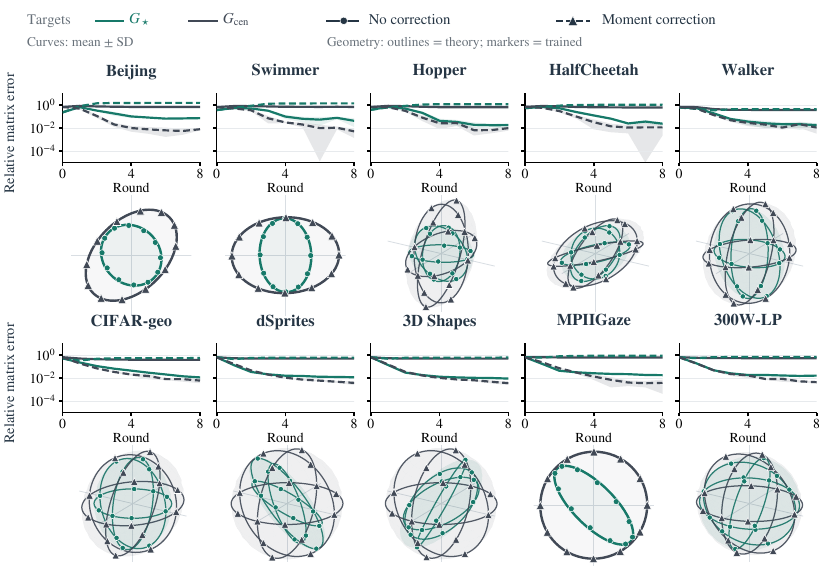}
 \caption{\textbf{Ordinary training with and without moment correction.} The design matches Figures~\ref{fig:dnn-selection} and~\ref{fig:dnn-selection-appendix}, with $\rho=0$. Curves show seed means $\pm1$ standard deviation; geometry uses seed 0. Both variants retain the feature and head penalties.}
 \label{fig:dnn-ordinary}
\end{figure}

\begin{figure}[!t]
 \centering
 \includegraphics[width=\linewidth]{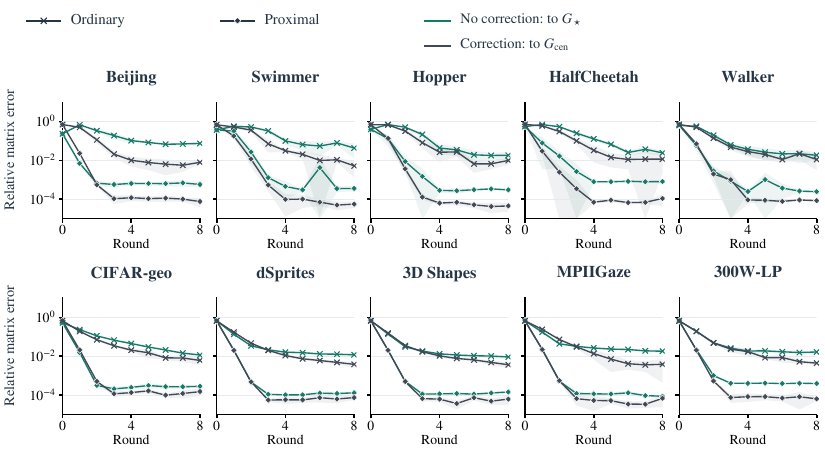}
 \caption{\textbf{Direct comparison of ordinary and proximal training.} Green curves measure training without correction against $G_\star$; graphite gray curves measure training with moment correction against $G_{\cen}$. Within each color, crosses and filled diamonds denote ordinary and proximal training; both use solid lines. Curves use the same five seeds per setting as Figures~\ref{fig:dnn-selection} and~\ref{fig:dnn-ordinary}, with means and one-standard-deviation bands. All panels share their axis limits.}
 \label{fig:dnn-protocol-comparison}
\end{figure}

\subsection{Removing the mean contribution}
\label{app:moment-dnn}
To test the gap decomposition of Proposition~\ref{sub:channels}, we remove the between-client mean term by centering each client's targets and measure the remaining gap under ordinary training. Centering retains each client's covariance. We evaluate this change on all ten settings, using ResMLP for the tabular datasets and ResNet-18 for the image datasets, with three seeds per setting. The pooled reference then becomes $G_{\mathrm{within}}$, while $G_\star$ remains unchanged. In all 30 dataset--seed combinations, $G_{\mathrm{within}}-G_8$ is positive definite, consistent with a remaining covariance and averaging gap. The median direction error to $G_\star$ is $9.48\times10^{-3}$, and the median ratio of the observed gap trace to $\Tr(\mathcal M_\Sigma+\mathcal M_A)$ is $1.19$ (Figure~\ref{fig:dnn-extensions}).

Centering removes only the mean term, so the covariance and averaging terms remain together in this comparison. Two of the 30 runs have a direction error to $G_\star$ above $0.05$, and in four runs this error differs from that of the uncentered control by more than $0.02$.

\begin{figure}[!htbp]
 \centering
 \includegraphics[width=\linewidth]{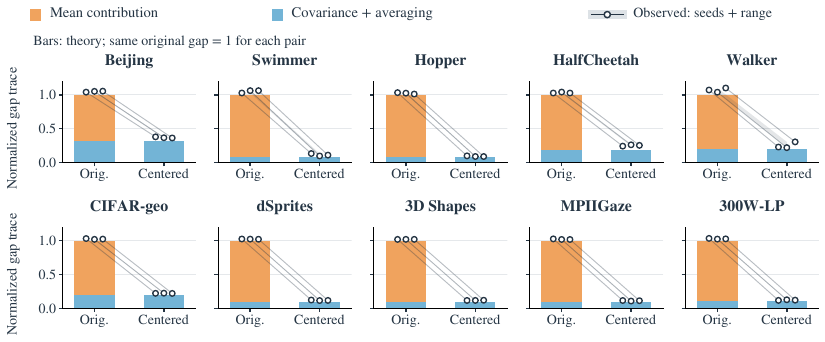}
 \caption{\textbf{Removing the mean contribution leaves a residual gap.} Stacked bars separate the predicted mean contribution $\Tr(\mathcal M_\mu)$ from $\Tr(\mathcal M_\Sigma+\mathcal M_A)$ before and after centering each client's targets. Both conditions use the original theoretical gap trace as denominator, computed per seed before averaging the components. Open markers show all three paired seeds, joined by thin lines; pale shading spans their range. Their observed gaps are $\Tr(G_{\cen}-G_8)$ before centering and $\Tr(G_{\mathrm{within}}-G_8)$ afterward, on the same scale.}
 \label{fig:dnn-extensions}
\end{figure}

\subsection{Local epochs and the endpoint}
\label{app:esweep-dnn}
To see how close ordinary training comes to $G_\star$ as local training grows, we vary the number of local epochs $E$ per round. The local-budget sweep covers all ten settings with $E\in\{50,100,200,400,1000,2000,4000\}$ and seeds $\{0,1,2\}$, keeping the other training parameters fixed.
We compare the direction error of the server Gram $G_8$ to $G_\star$ across local training budgets (Figure~\ref{fig:dnn-esweep}). The median over setting-level medians is $4.57\times10^{-1}$ at $E=50$, $1.76\times10^{-1}$ at $E=400$, $4.33\times10^{-2}$ at $E=1000$ and $1.39\times10^{-2}$ at $E=4000$. The first tested budget with median direction error below $0.1$ is $E=400$ for Beijing and CIFAR-geo, $E=1000$ for Walker and the other four image settings, and $E=2000$ for Swimmer, Hopper and HalfCheetah.
All budgets use eight rounds, so total local training grows with $E$, by a factor of 80 from $E=50$ to $E=4000$. A similar qualitative trend toward the BW prediction is observed in the profiled UFM experiments in Appendix~\ref{sub:num-finite}, although the update rules and evaluation protocols differ.

\begin{figure}[!htbp]
 \centering
 \includegraphics[width=\linewidth]{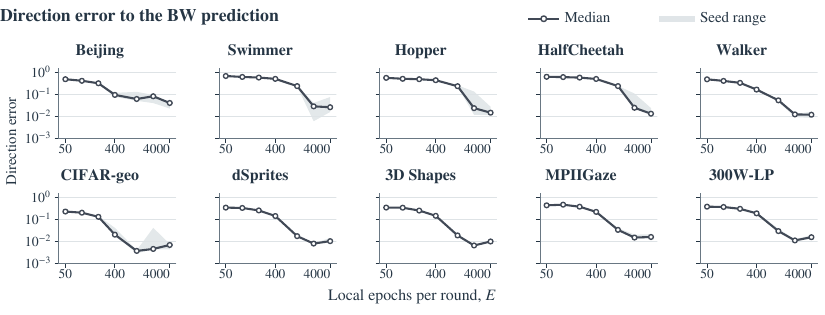}
 \caption{\textbf{Local epochs and the endpoint.} Direction error $d(G_8,G_\star)$ versus local epochs $E$, with shared logarithmic axes. Curves are medians over seeds $\{0,1,2\}$ and bands span these seeds. Top row: ResMLP; bottom row: ResNet-18.}
 \label{fig:dnn-esweep}
\end{figure}

\subsection{Dependence on the proximal weight}
\label{app:proximal-weight}
To see how the proximal weight affects agreement with the selection rule and with the local optima, we vary $\rho$ and measure the endpoint error, the distance to the selected heads and the relative original-objective gap. The weight sweep uses all ten settings with seeds $\{0,1,2\}$, giving the same 30 dataset--seed combinations at every weight. It compares $\rho=0$ with $\rho\in\{10^{-5},10^{-4},10^{-3},10^{-2},10^{-1},1\}$, and every run uses eight rounds and $E=4000$ local epochs.

At $\rho=10^{-3}$, the distance to the selected heads and the endpoint relative matrix error $e_F(G_8,G_\star)$ both decrease relative to $\rho=0$ in all 30 dataset--seed combinations. The median final distance $\delta_{\rm last}$ falls from $1.91\times10^{-1}$ to $5.33\times10^{-3}$, while the median relative original-objective gap rises from $1.43\times10^{-5}$ to $4.05\times10^{-5}$. Increasing $\rho$ further moves the clients away from their unpenalized optima and worsens endpoint agreement (Table~\ref{tab:dnn-anchor}, Figure~\ref{fig:dnn-rho-sweep}). At $\rho=10^{-5}$, the median $e_F$ is $7.38\times10^{-3}$, compared with $1.71\times10^{-2}$ at $\rho=0$ and $3.08\times10^{-4}$ at $\rho=10^{-3}$. In Lemma~\ref{sub:selection-main}, the selected head arises from minimizing each penalized objective globally and then letting $\rho\downarrow0$, whereas with finite local training a very small weight has a weaker effect. The medians of $\delta_{\rm last}$ and $e_F$ are therefore U-shaped in $\rho$, with minima at $\rho=10^{-3}$ (Table~\ref{tab:dnn-anchor}).

\begin{table}[!htbp]
 \centering\small
 \caption{\textbf{Proximal weight and local fitting.} Medians at the final round over the same 30 dataset--seed combinations at each weight. Relative gaps use the original loss without the proximal regularizer and first take the median across clients. Bold marks the smallest value in each column.}
 \label{tab:dnn-anchor}
 \begin{tabular}{@{}lccc@{}}
\toprule
Proximal weight $\rho$ & \shortstack{Distance to selected heads\\$\delta_{\rm last}$} & Relative original-objective gap & $e_F(G_8,G_\star)$ \\
\midrule
$0$ & $1.91\times10^{-1}$ & {\boldmath$1.43\times10^{-5}$} & $1.71\times10^{-2}$ \\
$10^{-5}$ & $1.19\times10^{-1}$ & $1.49\times10^{-5}$ & $7.38\times10^{-3}$ \\
$10^{-4}$ & $3.73\times10^{-2}$ & $1.53\times10^{-5}$ & $9.14\times10^{-4}$ \\
$10^{-3}$ & {\boldmath$5.33\times10^{-3}$} & $4.05\times10^{-5}$ & {\boldmath$3.08\times10^{-4}$} \\
$10^{-2}$ & $3.31\times10^{-2}$ & $1.40\times10^{-3}$ & $4.17\times10^{-3}$ \\
$10^{-1}$ & $1.32\times10^{-1}$ & $1.81\times10^{-2}$ & $4.26\times10^{-2}$ \\
$1$ & $3.56\times10^{-1}$ & $1.57\times10^{-1}$ & $3.80\times10^{-1}$ \\
\bottomrule
\end{tabular}

\end{table}

\begin{figure}[!htbp]
 \centering
 \includegraphics[width=\linewidth]{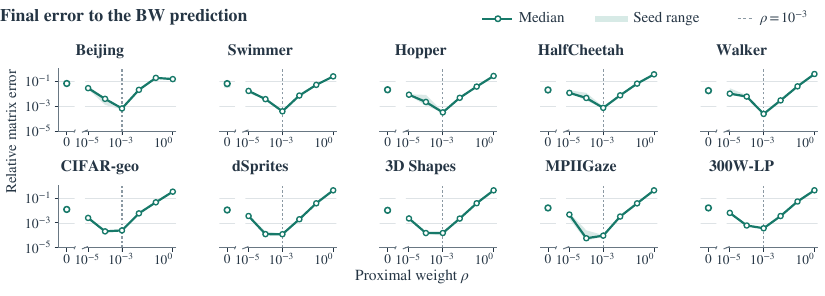}
 \caption{\textbf{Weight of the proximal regularizer.} Endpoint relative matrix error $e_F(G_8,G_\star)$ with $E=4000$. In each setting, curves are medians over seeds $\{0,1,2\}$ and pale bands span these seeds. The zero-weight baseline is separate from the positive logarithmic axis; dashed guides mark $\rho=10^{-3}$.}
 \label{fig:dnn-rho-sweep}
\end{figure}

\subsection{Selection by alignment}
\label{app:alignment}
Besides the proximal regularizer, the selection rule can be imposed by aligning each trained head to the broadcast head. We use this second implementation to examine (1) whether the endpoint again reaches $G_\star$, (2) whether moment correction still recovers $G_{\cen}$, and (3) how much averaging heads in mismatched coordinates raises the training objective. Each round has two stages: local training followed by an orthogonal alignment to the broadcast head. Local training uses the objective in Appendix~\ref{app:return-protocol} with $\rho=0$, either with or without moment correction. Let $\widehat W_{m,t}$ and $\widehat h_{m,t}(x)$ be the resulting head and feature vector. Client $m$ computes
\[
 \begin{gathered}
 Q_{m,t}\in\underset{Q\in\mathbb R^{P\times P}:\,Q^\top Q=I_P}{\arg\min}
       \|\widehat W_{m,t}Q-W_t\|_F^2,\\
 \widetilde W_{m,t}=\widehat W_{m,t}Q_{m,t},\qquad
 \widetilde h_{m,t}(x)=Q_{m,t}^\top\widehat h_{m,t}(x).
 \end{gathered}
\]
The client applies the feature transformation to its final linear feature layer, including its bias, and retains that layer for the next round. The head bias is unchanged by alignment. This joint transformation preserves the local predictions, head Gram matrix and feature and head penalties. The server then averages the aligned heads and their biases with weights $p_m=N_m/N$.

The implementation computes the aligned head from $C\times C$ Gram matrices and obtains $Q_{m,t}$ by SVD-based orthogonal completion. When both the returned and broadcast heads have full row rank, the aligned head is the unique head with the observed Gram matrix that is closest to $W_t$ in Frobenius norm, by the projection formula in Appendix~\ref{sub:selection-proof}. When local training reaches the UFM-optimal Gram $G_m$, this gives the selected head $\Pi_m(W_t)$ in Lemma~\ref{sub:selection-main}(i).

\paragraph{Geometry over communication rounds.}
We compare ordinary, proximal ($\rho=10^{-3}$) and aligned training on six settings: Beijing, Swimmer, Hopper and HalfCheetah with ResMLP, and 3D Shapes and CIFAR-geo with ResNet-18. Each setting uses all five seeds $\{0,1,2,3,4\}$, giving 30 matched comparisons. Runs use eight rounds and $E=4000$ local epochs, with the other training parameters from Appendix~\ref{app:return-protocol}. Alignment reduces the distance to the selected heads, the trajectory error and the endpoint error to $G_\star$ in all 30 comparisons with ordinary training, and brings the median endpoint error to $1.10\times10^{-4}$, compared with $3.24\times10^{-4}$ under proximal training and $1.86\times10^{-2}$ under ordinary training (Table~\ref{tab:dnn-alignment}). Relative to proximal training, alignment reduces trajectory error in all 30 comparisons and endpoint error to $G_\star$ in 29.

The moment-corrected comparison uses the same six settings with seeds $\{0,1,2\}$, giving 18 matched comparisons. Under exact UFM optimization, the corrected local optima share $G_{\cen}$, and alignment to the common broadcast selects the same head at every client, which yields the one-round recovery in Proposition~\ref{sub:correction}. \textbf{\boldmath In the trained networks, corrected alignment reaches a median relative error of $8.38\times10^{-5}$ to $G_{\cen}$ after the first round, whereas corrected proximal training needs three rounds to reach a comparable error} ($2.38\times10^{-2}$, $6.54\times10^{-4}$ and $9.84\times10^{-5}$ after rounds 1, 2 and 3).

\begin{table}[!htbp]
 \centering\small
 \setlength{\tabcolsep}{4pt}
 \caption{\textbf{Alignment, proximal training and ordinary training.} Medians over 30 matched dataset--seed combinations without correction and 18 with moment correction. The distance $\delta_{\rm last}$ and trajectory error are defined in Appendix~\ref{app:local-screen}. Without correction, endpoint error compares $G_8$ with $G_\star$; corrected errors compare $G_t$ with $G_{\cen}$. Bold marks the smallest value in each column.}
 \label{tab:dnn-alignment}
 \begin{tabular}{lccccc}
\toprule
 & \multicolumn{3}{c}{Without correction ($n=30$)} & \multicolumn{2}{c}{Corrected ($n=18$)} \\
\cmidrule(lr){2-4}\cmidrule(lr){5-6}
Training & $\delta_{\rm last}$ & Trajectory $e_F$ & Endpoint $e_F$ & Round 1 $e_F$ & Round 8 $e_F$ \\
\midrule
Ordinary & $1.87\times10^{-1}$ & $6.75\times10^{-2}$ & $1.86\times10^{-2}$ & $5.20\times10^{-1}$ & $7.07\times10^{-3}$ \\
Proximal & $6.35\times10^{-3}$ & $3.63\times10^{-4}$ & $3.24\times10^{-4}$ & $2.38\times10^{-2}$ & $7.06\times10^{-5}$ \\
Aligned & {\boldmath$1.54\times10^{-4}$} & {\boldmath$1.12\times10^{-4}$} & {\boldmath$1.10\times10^{-4}$} & {\boldmath$8.38\times10^{-5}$} & {\boldmath$6.27\times10^{-5}$} \\
\bottomrule
\end{tabular}

\end{table}

\paragraph{Cost of mismatched coordinates.}
To isolate the effect of solution multiplicity on one server update, we aggregate the same first-round local models twice, once as trained and once after alignment, and evaluate the weighted training objective~\eqref{sub:pooled-loss} with the corresponding private features and the same head bias. The same six settings, with five seeds each and common or independently initialized backbones, give 60 runs; all clients start from the same head and train for $E=4000$ epochs without correction or a proximal regularizer. Alignment lowers the objective in all 60 runs and removes a median of $4.6\%$ to $23.7\%$ of it across settings, similarly for the two initializations (Table~\ref{tab:dnn-alignment-aggregation}). Averaging heads in mismatched coordinates thus raises the training objective within a single round.

\begin{table}[!htbp]
 \centering\small
 \caption{\textbf{Cost of mismatched coordinates in one aggregation.} The same first-round local models are averaged twice, with the heads as trained and after alignment to the broadcast head, giving training objectives~\eqref{sub:pooled-loss} $L_{\rm raw}$ and $L_{\rm al}$. Each entry is the median over five seeds of $(L_{\rm raw}-L_{\rm al})/L_{\rm raw}$, the fraction of the objective that alignment removes. Columns give the initialization of the private backbones; all clients start from the same head. Tabular settings use ResMLP and image settings use ResNet-18.}
 \label{tab:dnn-alignment-aggregation}
 \begin{tabular}{lcc}
\toprule
 & \multicolumn{2}{c}{Fraction of the objective removed by alignment} \\
\cmidrule(lr){2-3}
Setting & Common backbone init. & Independent backbone init. \\
\midrule
\multicolumn{3}{l}{\textit{Tabular datasets}} \\
Beijing & 15.7\% & 15.4\% \\
Swimmer & 10.1\% & 10.5\% \\
Hopper & 17.7\% & 18.1\% \\
HalfCheetah & 23.6\% & 23.7\% \\
\midrule
\multicolumn{3}{l}{\textit{Image datasets}} \\
3D Shapes & 4.6\% & 6.0\% \\
CIFAR-geo & 7.0\% & 8.3\% \\
\bottomrule
\end{tabular}

\end{table}

\FloatBarrier

\subsection{Architecture, data size and full-model averaging}
\label{app:dnn-scope}
To test whether the agreement with $G_\star$ depends on the backbone, the amount of training data or the averaging scheme, we repeat the experiments with other backbones, larger training sets and full-model averaging. Table~\ref{tab:robustness} first compares network architectures and data sizes, then full-model averaging.

\paragraph{Architecture and data size.}
The residual CNN and ResNet-34 comparisons use ordinary training on all five image datasets with seeds $\{0,1,2\}$. The larger-data setting, also with ordinary training and ResNet-18, increases 3D Shapes from 12,000 to 48,000 training samples, keeping its eight clients, and CIFAR-geo from 12,000 to 40,000 training samples partitioned into eight clients by target projection. The corresponding main CIFAR-geo experiment uses four designed client groups. Both larger-data experiments use $E=1000$ local epochs per round, compared with $E=4000$ in the main experiments. In all three comparisons, the median direction error to $G_\star$ lies between $9.55\times10^{-3}$ and $1.45\times10^{-2}$, and every run is closer in direction to $G_\star$ than to $G_{\cen}$ (Table~\ref{tab:robustness}).

The comparison between ordinary and corrected training in Appendix~\ref{app:proximal-results} also covers the five residual-CNN settings, with three seeds each. All 65 ordinary-training runs, 50 from the main experiments and 15 with the residual CNN, are closer in direction to $G_\star$ than to $G_{\cen}$. For the residual CNN, moment correction reduces the median direction error to $G_{\cen}$ from $5.24\times10^{-1}$ to $1.93\times10^{-3}$.

\begin{table}[!htbp]
\centering
\small
\setlength{\tabcolsep}{4pt}
\caption{\textbf{Architecture, data size and full-model averaging.} Median direction errors over all runs in each comparison. Closer to $G_\star$ is the percentage of runs closer in direction to $G_\star$ than to $G_{\cen}$; corrected runs target $G_{\cen}$. Architecture comparisons use $E=4000$ unless specified; other training conditions are described in the text.}\label{tab:robustness}
\begin{tabular}{@{}p{.43\linewidth}ccc@{}}
\toprule
Training / architecture & To $G_\star$ & To $G_{\cen}$ & Closer to $G_\star$ \\
\midrule
\multicolumn{4}{@{}l}{\textit{Architecture and data}} \\
Residual CNN & $1.45\times10^{-2}$ & $5.24\times10^{-1}$ & 100\% \\
Residual CNN + correction & $5.05\times10^{-1}$ & $1.93\times10^{-3}$ & 0\% \\
ResNet-34 & $1.29\times10^{-2}$ & $5.19\times10^{-1}$ & 100\% \\
Larger data, $E=1000$ & $9.55\times10^{-3}$ & $5.63\times10^{-1}$ & 100\% \\
\midrule
\multicolumn{4}{@{}l}{\textit{Full-model averaging}} \\
$E=400$ & $4.41\times10^{-1}$ & $1.82\times10^{-1}$ & 27\% \\
$E=4000$ & $2.14\times10^{-2}$ & $6.20\times10^{-1}$ & 100\% \\
$E=4000$ + correction & $6.22\times10^{-1}$ & $8.08\times10^{-3}$ & 0\% \\
\bottomrule
\end{tabular}
\end{table}

\paragraph{Full-model averaging.}
Under full-model averaging, the server averages the backbone as well as the head, so the clients also share their feature extractors. These experiments use the five tabular datasets with ResMLP and three seeds each. With $E=400$, the endpoint is far from both $G_\star$ and $G_{\cen}$, with median direction errors of $4.41\times10^{-1}$ and $1.82\times10^{-1}$. With $E=4000$, the median direction error to $G_\star$ falls to $2.14\times10^{-2}$, and every run is closer in direction to $G_\star$ than to $G_{\cen}$. Adding moment correction at $E=4000$ reduces the median direction error to $G_{\cen}$ from $6.20\times10^{-1}$ to $8.08\times10^{-3}$ (Figure~\ref{fig:dnn-fullmodel-paired}). With sufficient local training, the endpoint under full-model averaging therefore stays close to $G_\star$, and moment correction still moves it to $G_{\cen}$.

\begin{figure}[!htbp]
 \centering
 \includegraphics[width=\linewidth]{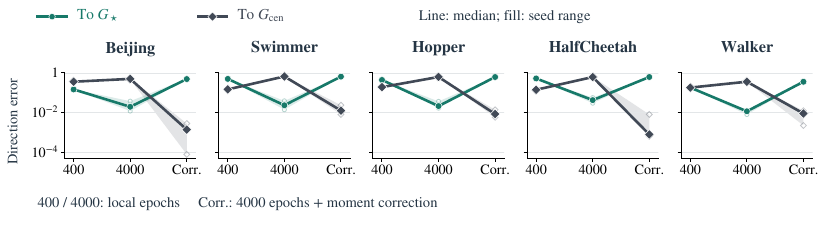}
 \caption{\textbf{Local training and moment correction under full-model averaging.} Each panel shows one tabular dataset with ResMLP and compares the same three seeds across ordinary training with $E=400$, ordinary training with $E=4000$, and corrected training with $E=4000$. Colors distinguish direction error to $G_\star$ and $G_{\cen}$. Curves show medians; pale bands span the three seeds. The first transition changes the local budget; the second adds correction at a fixed budget.}
 \label{fig:dnn-fullmodel-paired}
\end{figure}
\FloatBarrier

\subsection{Test error}
\label{app:prediction}
We next examine whether the proximal regularizer and the moment correction change the prediction error on test samples. This question lies outside the UFM analysis. The UFM treats the features of the training samples as free variables and describes how the training objective is optimized; it does not model the features that a backbone produces for unseen inputs, and hence does not predict test error. The preceding experiments show that this analysis predicts the endpoint of training across datasets and architectures. We add a comparison of test error for two reasons: (1) test error ultimately matters for any training procedure, and (2) the two interventions may affect test error differently:

\noindent\begin{minipage}{\linewidth}
\begin{itemize}
\item The moment correction moves the head Gram to the optimum of centralized training, using target moments that the clients share with the server, and may therefore lower test error.
\item The proximal regularizer adds an extra constraint to each local update, and it is unclear whether this constraint impairs generalization.
\end{itemize}
\end{minipage}

The comparison uses the main-experiment runs on the five image datasets with ResNet-18 and seeds $\{0,1,2\}$, 15 dataset--seed combinations in total. Ordinary training, proximal training and proximal training with moment correction share the recipe $\lambda_H=\lambda_W=0.01$, zero weight decay, eight rounds and $E=4000$. All three keep the feature and head penalties; proximal training adds $\rho=10^{-3}$, and the corrected procedure also adds $C_m$. Prediction error is evaluated as in Appendix~\ref{app:local-screen}, with test samples assigned to clients as in Appendix~\ref{app:dnn-datasets}, and MSE changes are paired within each dataset--seed combination.

The proximal regularizer tightens agreement with $G_\star$, and adding correction moves the endpoint to $G_{\cen}$ (Table~\ref{tab:dnn-prediction}). Both interventions lower training MSE in all 15 comparisons, with median changes of $-0.2\%$ and $-2.0\%$. Proximal training changes test MSE by less than $1.5\%$ in every comparison and lowers it in 9 of 15. Adding correction lowers test MSE in 13 of 15 comparisons, with a median change of $-1.5\%$; the mean change ranges from $-2.3\%$ on dSprites to $+0.3\%$ on MPIIGaze (Figure~\ref{fig:dnn-prediction-settings}). On these datasets, the proximal regularizer thus constrains training while keeping test error within $1.5\%$ of ordinary training, and the moment correction lowers test error slightly on average.

\begin{table}[!htbp]
 \centering\small
 \setlength{\tabcolsep}{4pt}
 \caption{\textbf{Geometry and prediction under one training recipe.} Medians over the 15 image dataset--seed combinations. Geometry columns are relative matrix errors. The loss is the MSE, and loss changes are paired against ordinary training before taking medians. \textbf{Negative changes indicate lower loss, which is better.}}
 \label{tab:dnn-prediction}
 \begin{tabular}{@{}lcccc@{}}
\toprule
Training & $e_F(G_8,G_\star)$ & $e_F(G_8,G_{\cen})$ & Train loss change & Test loss change \\
\midrule
Ordinary & $1.37\times10^{-2}$ & $5.04\times10^{-1}$ & -- & -- \\
Proximal & $1.54\times10^{-4}$ & $4.96\times10^{-1}$ & $-0.2\%$ & $-0.1\%$ \\
Proximal + moment correction & $6.23\times10^{-1}$ & $7.03\times10^{-5}$ & $-2.0\%$ & $-1.5\%$ \\
\bottomrule
\end{tabular}

\end{table}

\begin{figure}[H]
 \centering
 \includegraphics[width=\linewidth]{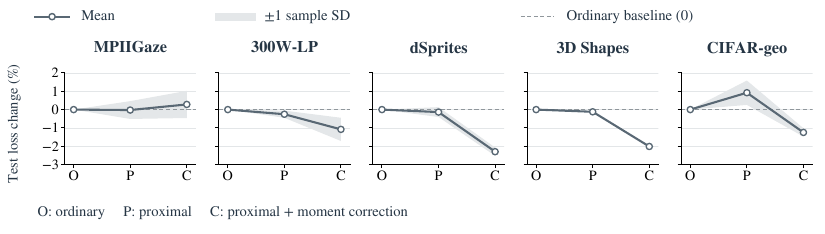}
 \caption{\textbf{Paired test-MSE changes on the five image datasets.} Values are $100(\mathrm{MSE}/\mathrm{MSE}_{\mathrm O}-1)$, with the same seed's ordinary-training MSE as denominator for all procedures; O, P and C denote ordinary training, proximal training, and proximal training with moment correction. Curves show means over seeds $\{0,1,2\}$, and bands span one sample standard deviation; the dashed line is zero. \textbf{Negative values indicate lower loss, which is better.}}
 \label{fig:dnn-prediction-settings}
\end{figure}

\FloatBarrier

\section{Supplementary analysis}
\label{sub:supplementary-analysis}

\subsection{Finite-step dynamics of the profiled UFM}
\label{sub:finite-proof}

We study a federated update in which each client performs gradient
descent on its profiled objective $F_m(W)=F(W;\Sigma_m)$.
Each gradient step corresponds to first minimizing the local loss
over the features and bias at the current head, then taking a
full-batch gradient step on the head.
We call this a \emph{profiled step}.

At each round, all clients start from the broadcast $W_t$,
take $E$ profiled steps with learning rate $\eta$, and return
their heads for weighted averaging.
Since each profiled step uses the full local batch, $E$ also counts
the local epochs per round, as in the DNN experiments.
We write $h=\eta E$ and denote the server Gram matrix by
$G_t=W_tW_t^\top$.

\paragraph{One local step per round.}
For $E=1$, the server update is
\[
W_{t+1}
=W_t-\eta\sum_m p_m\nabla F(W_t;\Sigma_m)
=W_t-\eta\nabla F(W_t;\Sigma_{\mathrm{within}}),
\]
where the second equality follows from the affine dependence
of $F(W;T)$ on $T$.
Thus the server performs gradient descent on the averaged
profiled objective
\[
\bar F(W):=\sum_m p_mF_m(W)
=F(W;\Sigma_{\mathrm{within}}).
\]
In the fully active regime, its optimal Gram matrix is
\[
G_{\mathrm{within}}=\phi(\Sigma_{\mathrm{within}}).
\]
Every head with this Gram matrix is a fixed point of the
one-step update. We use $G_{\mathrm{within}}$ as the reference
for studying multiple local steps.

\paragraph{Comparison with joint updates.}
For comparison, suppose each client takes one simultaneous
full-batch gradient step on $\mathcal L_m$ in $(W,b,H_m)$,
starting from the common head $(W,b)$.
The server averages the updated heads, while the features remain local.

The averaged updates of $W$ and $b$ equal those of gradient
descent on the pooled objective \eqref{sub:pooled-loss}.
Since $\nabla_{H_m}\mathcal L_{\cen}
=p_m\nabla_{H_m}\mathcal L_m$, the feature update is
\[
H_m^+
=H_m-\eta\nabla_{H_m}\mathcal L_m
=H_m-\frac{\eta}{p_m}\nabla_{H_m}\mathcal L_{\cen}.
\]
Thus the entire round corresponds to gradient descent on
$\mathcal L_{\cen}$, with learning rate $\eta$ for $(W,b)$
and $\eta/p_m$ for $H_m$.
If these joint updates converge to a pooled global minimizer,
their head Gram converges to $G_{\cen}$.
The profiled update instead follows $\bar F$, whose optimal
Gram is $G_{\mathrm{within}}$.
The objectives underlying the two procedures therefore have
generally different optimal Gram matrices, even with one
local step per round.

\paragraph{Displacement under multiple local steps.}
We now study how the Gram fixed point changes when $E>1$.
To state the result, write the gradient of the profiled objective as
\[
\nabla F(W;T)=2\Lambda_T(G)W,
\qquad G=WW^\top,
\]
where
\[
\Lambda_T(G)
=\frac{\lambda_W}{2}I_C
-\frac{\lambda_H}{2}
(G+\lambda_H I_C)^{-1}T(G+\lambda_H I_C)^{-1}.
\]
We use the abbreviations
\[
\Lambda_m=\Lambda_{\Sigma_m},
\qquad
\bar\Lambda=\Lambda_{\Sigma_{\mathrm{within}}},
\qquad
R=\Sigma_{\mathrm{within}}^{1/2},
\qquad
\alpha=\sqrt{\lambda_H/\lambda_W},
\]
so that $G_{\mathrm{within}}=\alpha R-\lambda_H I_C$.

Let $f(G)=\bar F(W)$ whenever $WW^\top=G$, and define
\begin{equation}
\begin{aligned}
\psi(G)&=2\sum_m p_m\Tr\!\bigl(\Lambda_m(G)G\Lambda_m(G)\bigr),\\
\mathcal H[X]&=\frac{\lambda_W}{2\alpha}(R^{-1}X+XR^{-1}).
\end{aligned}
\label{sub:drift-objects}
\end{equation}
These quantities satisfy
\[
\psi(WW^\top)=\frac12\sum_m p_m\norm{\nabla F_m(W)}_F^2,
\qquad
\mathcal H=\operatorname{Hess}f(G_{\mathrm{within}}).
\]
Thus $\psi$ measures the weighted squared norms of the local
head gradients, while $\mathcal H$ is the Hessian of the
averaged objective with respect to the Gram matrix.
The operator $\mathcal H$ is positive definite on
$\operatorname{Sym}(C)$, the space of symmetric $C\times C$ matrices.

For sufficiently small $h$, the following theorem establishes
a locally attracting Gram fixed point near $G_{\mathrm{within}}$
and gives its first-order displacement from $G_{\mathrm{within}}$.

\begin{theorem}[Fixed point and local convergence of profiled updates]
\label{sub:finite}
Assume the fully active conditions of
Section~\ref{sub:formulation}.
For an integer $E\ge1$ and learning rate $\eta>0$, define
the server update
\[
\Phi_{\eta,E}(W)
=\sum_m p_m\mathrm{GD}_{\eta}^{E}[F_m](W),
\]
where $\mathrm{GD}_{\eta}^{E}[F_m](W)$ denotes $E$ gradient
steps on $F_m$, starting from $W$.

There exist constants $r,\varepsilon_0,K>0$ such that,
for every $P\ge C$ and every $(\eta,E)$ satisfying
$h=\eta E\le\varepsilon_0$, the following hold.

\begin{enumerate}
\item \emph{Gram dynamics and a unique nearby fixed point.}
The next Gram matrix depends only on the current Gram matrix,
so the update defines a map
\[
G_{t+1}=\Phi_{\eta,E}^{G}(G_t).
\]
This map has exactly one fixed point in
$\norm{G-G_{\mathrm{within}}}_F\le r$.
We denote this positive-definite fixed point by $G_{\eta,E}$.

\item \emph{Displacement from the one-step fixed point.}
Define
\[
\Delta
=\frac12\mathcal H^{-1}
\!\left[\nabla\psi(G_{\mathrm{within}})\right],
\]
where $\psi$ and $\mathcal H$ are given in
\eqref{sub:drift-objects}.
Then
\[
\norm{G_{\eta,E}-G_{\mathrm{within}}-\eta(E-1)\Delta}_F
\le K(\eta E)^2.
\]
For $E=1$, the fixed point is exactly
$G_{\eta,1}=G_{\mathrm{within}}$.

\item \emph{Convergence from nearby initializations.}
If $G_0$ is sufficiently close to $G_{\eta,E}$, then for all
$t\ge0$,
\[
\norm{G_t-G_{\eta,E}}_F
\le
\left(1-\frac12h\theta_{\min}\right)^t
\norm{G_0-G_{\eta,E}}_F,
\]
where
\[
\theta_{\min}
=2\lambda_W
\left(
1-\sqrt{\frac{\lambda_H\lambda_W}
{\lambda_{\min}(\Sigma_{\mathrm{within}})}}
\right)>0.
\]
\end{enumerate}

The constants $r,\varepsilon_0,K$ depend only on the client
covariances, weights, regularization parameters and $C$;
they are independent of $\eta$, $E$ and $P$.
The convergence statement concerns the Gram matrices $G_t$.
No convergence of the heads $W_t$ is asserted.
\end{theorem}

\subsection{Proof of Theorem~\ref{sub:finite}}
\label{sub:finite-derivation}

\begin{proof}
We write $\Phi=\Phi_{\eta,E}$ for the head update and
$\Phi^G=\Phi^G_{\eta,E}$ for its induced Gram map.
All remainder estimates below hold uniformly in a fixed neighborhood
of $G_{\mathrm{within}}$, for sufficiently small $h=\eta E$.
Their constants may depend on the client covariances, weights,
regularization parameters and $C$, but not on $\eta$, $E$ or $P$.
For linear operators on $\operatorname{Sym}(C)$, we use the norm
induced by the Frobenius norm and denote the identity operator by
$\operatorname{Id}$.

\medskip
\noindent\emph{1. Exact Gram updates and the local expansion.}
Fix a broadcast head $W$ with Gram $G=WW^\top$.
For client $m$, let $W_{m,k}$ be the head after $k$ local steps and
let $G_{m,k}=W_{m,k}W_{m,k}^\top$, with $W_{m,0}=W$ and $G_{m,0}=G$.
Define
\[
A_{m,k}=I_C-2\eta\Lambda_m(G_{m,k}).
\]
The head update is $W_{m,k+1}=A_{m,k}W_{m,k}$.
Since $A_{m,k}$ is symmetric, its Gram satisfies
\begin{equation}
G_{m,k+1}=A_{m,k}G_{m,k}A_{m,k}.
\label{sub:profiled-local-gram}
\end{equation}
Thus each local Gram trajectory depends only on the initial Gram $G$.
The returned head is $W_{m,E}=M_m(G)W$, where
\[
M_m(G)=A_{m,E-1}\cdots A_{m,0}.
\]
Writing $M(G)=\sum_m p_mM_m(G)$, the server update and its Gram are
\[
\Phi(W)=M(G)W,
\qquad
\Phi^G(G)=M(G)GM(G)^\top.
\]
This establishes that the server Gram evolves independently of the
choice of head factor $W$.

Choose two closed Frobenius balls centered at
$G_{\mathrm{within}}$, with radii $0<r_1<r_2$, both contained
in the positive-definite cone.
Let the initial Gram satisfy
$\norm{G-G_{\mathrm{within}}}_F\le r_1$.
On the larger ball, smoothness and compactness give uniform
bounds on the required derivatives and on each local increment:
$\norm{G_{m,k+1}-G_{m,k}}_F\le c\eta$ for sufficiently small $\eta$.
If $h=\eta E$ is small enough that $ch<r_2-r_1$, induction gives
\[
\norm{G_{m,k}-G}_F\le c\eta k,
\qquad
\norm{G_{m,k}-G_{\mathrm{within}}}_F
\le r_1+c\eta k<r_2,
\qquad 0\le k\le E.
\]
Thus all local iterates remain in the larger ball, where these
uniform bounds apply.

The estimates can be made independent of the feature width.
Indeed, write $W=G^{1/2}V$, where $VV^\top=I_C$.
Local gradient descent from $W$ is the corresponding iteration from
$G^{1/2}$ followed by right multiplication by $V$.
Since $\norm{ZV}_F=\norm Z_F$ for every $C\times C$ matrix $Z$,
the expansion and its remainder can be established using square
heads and then transferred to every $P\ge C$.

For a fixed client, abbreviate $g_m=\nabla F_m$ and $W_k=W_{m,k}$.
Taylor expansion along the local iterates gives
\[
\begin{aligned}
W_k-W
&=-\eta k g_m(W)+O((\eta k)^2),\\
g_m(W_k)
&=g_m(W)-\eta kDg_m(W)[g_m(W)]+O((\eta k)^2).
\end{aligned}
\]
Substituting the second expansion into
$W_E=W-\eta\sum_{k=0}^{E-1}g_m(W_k)$ yields
\begin{equation}
\begin{aligned}
\mathrm{GD}_{\eta}^{E}[F_m](W)
=W-h\nabla F_m(W)
+\frac{\eta^2E(E-1)}2
  \nabla^2F_m(W)[\nabla F_m(W)]+O(h^3).
\end{aligned}
\label{sub:local-expansion}
\end{equation}
Here $\sum_{k=0}^{E-1}k=E(E-1)/2$, while the summed remainder is
bounded by a constant times
$\eta^3\sum_{k=0}^{E-1}k^2\le h^3/3$.
This gives a remainder bound uniform in $E$.

\medskip
\noindent\emph{2. The averaged update and its derivative.}
Define
\[
\Psi(W)=\frac12\sum_m p_m\norm{\nabla F_m(W)}_F^2
       =\psi(WW^\top).
\]
Its gradient is
\[
\nabla\Psi(W)
=\sum_m p_m\nabla^2F_m(W)[\nabla F_m(W)].
\]
Averaging \eqref{sub:local-expansion} therefore gives
\[
\Phi(W)
=W-h\nabla\bar F(W)
 +\frac{\eta^2E(E-1)}2\nabla\Psi(W)+O(h^3).
\]
Thus, up to an $O(h^3)$ remainder, one communication round agrees
with a gradient step of size $h$ on
$\bar F-\eta(E-1)\Psi/2$.
The identity $\nabla\bar F=\sum_m p_m\nabla F_m$ also gives
\[
\Psi(W)
=\frac12\norm{\nabla\bar F(W)}_F^2
 +\frac12\sum_m p_m
       \norm{\nabla F_m(W)-\nabla\bar F(W)}_F^2.
\]
The second term measures the dispersion of the client gradients
at the common head $W$.

For a symmetric matrix $A$, define the linear operator
$\mathcal S_A[X]=AX+XA$ on symmetric matrices.
Using $\nabla\bar F(W)=2\bar\Lambda(G)W$ and
$\nabla\Psi(W)=2\nabla\psi(G)W$, we obtain
\begin{equation}
\begin{aligned}
\Phi^G(G)
=G-2h\mathcal S_G[\bar\Lambda(G)]+\eta^2E(E-1)\mathcal S_G[\nabla\psi(G)]+4h^2\bar\Lambda(G)G\bar\Lambda(G)+O(h^3).
\end{aligned}
\label{sub:gram-expansion}
\end{equation}

We next estimate the derivative of the exact Gram map.
Differentiating the recurrence \eqref{sub:profiled-local-gram}
with respect to the initial Gram gives
$DG_{m,k}(G)=\operatorname{Id}+O(\eta k)$.
The product rule for $M_m(G)$ then gives
\[
DM_m(G)[X]
=-2\eta\sum_{k=0}^{E-1}
P_{m,>k}\,
D\Lambda_m(G_{m,k})[DG_{m,k}(G)[X]]\,
P_{m,<k},
\]
where
\[
P_{m,>k}=A_{m,E-1}\cdots A_{m,k+1},
\qquad
P_{m,<k}=A_{m,k-1}\cdots A_{m,0}.
\]
Empty products equal $I_C$. Both products are $I_C+O(h)$.
Using $G_{m,k}=G+O(\eta k)$ in these expressions and averaging
over clients gives
\[
\begin{aligned}
M(G)&=I_C-2h\bar\Lambda(G)+O(h^2),\\
DM(G)[X]&=-2hD\bar\Lambda(G)[X]+O(h^2)\norm X_F.
\end{aligned}
\]
Applying the product rule to $M(G)GM(G)^\top$ now yields
\begin{equation}
D\Phi^G(G)
=\operatorname{Id}
-2h\bigl(
\mathcal S_G\circ D\bar\Lambda(G)
+\mathcal S_{\bar\Lambda(G)}
\bigr)+O(h^2).
\label{sub:derivative-estimate}
\end{equation}
The remainder is uniform in the induced operator norm on a fixed
smaller neighborhood of $G_{\mathrm{within}}$.

\medskip
\noindent\emph{3. Existence and uniqueness of a nearby fixed point.}
For fixed $(\eta,E)$, define
\[
Q(G)=\frac{G-\Phi^G(G)}{2h},
\qquad
\mathcal T=\mathcal S_{G_{\mathrm{within}}}\circ\mathcal H.
\]
A fixed point of $\Phi^G$ is exactly a zero of $Q$.
At the reference Gram,
$\bar\Lambda(G_{\mathrm{within}})=0$ and
$D\bar\Lambda(G_{\mathrm{within}})=\mathcal H$.
Hence \eqref{sub:derivative-estimate} implies
\[
DQ(G)
=\mathcal T+O(\norm{G-G_{\mathrm{within}}}_F+h).
\]
The operators $\mathcal H$ and $\mathcal S_{G_{\mathrm{within}}}$
are diagonal in the symmetric matrix basis associated with an
orthonormal eigenbasis of $R$.
Their eigenvalues are positive, so their product $\mathcal T$
is self-adjoint, positive definite and invertible.

Consider the auxiliary map
\[
\mathcal K(G)=G-\mathcal T^{-1}Q(G).
\]
The estimate for $DQ$ gives
\[
D\mathcal K(G)
=\operatorname{Id}-\mathcal T^{-1}DQ(G)
=O\!\left(\norm{G-G_{\mathrm{within}}}_F+h\right).
\]
Also, \eqref{sub:gram-expansion} gives
$Q(G_{\mathrm{within}})=O(h)$, and hence
\[
\mathcal K(G_{\mathrm{within}})-G_{\mathrm{within}}
=-\mathcal T^{-1}Q(G_{\mathrm{within}})=O(h).
\]
Thus there are constants $c_1,c_2>0$ such that, uniformly
in the neighborhood above for sufficiently small $h$,
\[
\begin{aligned}
\norm{D\mathcal K(G)}_{\mathrm{op}}
\le c_1\bigl(\norm{G-G_{\mathrm{within}}}_F+h\bigr),
\quad
\norm{\mathcal K(G_{\mathrm{within}})-G_{\mathrm{within}}}_F
\le c_2h.
\end{aligned}
\]
Choose $r>0$ small enough that the closed ball
$\norm{G-G_{\mathrm{within}}}_F\le r$ lies in this neighborhood
and $c_1r\le1/4$.
Then choose $\varepsilon_0>0$ small enough that the estimates
above hold for $h\le\varepsilon_0$ and
\[
c_1\varepsilon_0\le\frac14,
\qquad
c_2\varepsilon_0\le\frac r2.
\]
For every such $h$, these choices ensure
\[
\sup_{\norm{G-G_{\mathrm{within}}}_F\le r}
\norm{D\mathcal K(G)}_{\mathrm{op}}\le\frac12,
\qquad
\norm{\mathcal K(G_{\mathrm{within}})-G_{\mathrm{within}}}_F
\le\frac r2.
\]

Since the ball is convex, the derivative bound implies
\[
\norm{\mathcal K(G)-\mathcal K(G')}_F
\le\frac12\norm{G-G'}_F
\]
for any two matrices $G,G'$ in the ball.
Moreover, for every $G$ in the ball,
\[
\begin{aligned}
\norm{\mathcal K(G)-G_{\mathrm{within}}}_F
&\le
\norm{\mathcal K(G)-\mathcal K(G_{\mathrm{within}})}_F
+\norm{\mathcal K(G_{\mathrm{within}})-G_{\mathrm{within}}}_F\\
&\le\frac12\norm{G-G_{\mathrm{within}}}_F+\frac r2
\le r.
\end{aligned}
\]
Thus $\mathcal K$ is a contraction from the closed ball into itself.
Banach's fixed-point theorem gives a unique fixed point in this
ball, denoted $G_{\eta,E}$.
Since
\[
\mathcal K(G)=G
\quad\Longleftrightarrow\quad
Q(G)=0
\quad\Longleftrightarrow\quad
\Phi^G(G)=G,
\]
this is also the unique fixed point of the Gram update in the ball.
The ball is contained in the positive-definite cone, so
$G_{\eta,E}\succ0$.

Finally, the fixed-point identity and the contraction bound give
\[
\norm{G_{\eta,E}-G_{\mathrm{within}}}_F
\le
\frac12\norm{G_{\eta,E}-G_{\mathrm{within}}}_F
+\norm{\mathcal K(G_{\mathrm{within}})-G_{\mathrm{within}}}_F.
\]
Rearranging yields
\[
\norm{G_{\eta,E}-G_{\mathrm{within}}}_F
\le2c_2h=O(h).
\]

\medskip
\noindent\emph{4. Displacement of the fixed point.}
The preceding estimate implies
$\bar\Lambda(G_{\eta,E})=O(h)$.
Evaluate \eqref{sub:gram-expansion} at $G_{\eta,E}$ and divide
the fixed-point equation by $2h$. This gives
\[
\begin{aligned}
\mathcal S_{G_{\eta,E}}[\bar\Lambda(G_{\eta,E})]
=\frac{\eta(E-1)}2
\mathcal S_{G_{\eta,E}}[\nabla\psi(G_{\eta,E})]+2h\bar\Lambda(G_{\eta,E})G_{\eta,E}
        \bar\Lambda(G_{\eta,E})+O(h^2).
\end{aligned}
\]
For $G\succ0$, the operator $\mathcal S_G$ multiplies the
$(i,j)$ entry by the sum of two eigenvalues of $G$, when
expressed in an eigenbasis of $G$.
These sums are bounded away from zero in our neighborhood,
so $\mathcal S_G^{-1}$ is uniformly bounded there.
The quadratic term
$2h\bar\Lambda(G_{\eta,E})G_{\eta,E}\bar\Lambda(G_{\eta,E})$
is $O(h^3)$ because
$\bar\Lambda(G_{\eta,E})=O(h)$.
Applying $\mathcal S_{G_{\eta,E}}^{-1}$ therefore gives
\[
\bar\Lambda(G_{\eta,E})
=\frac{\eta(E-1)}2\nabla\psi(G_{\eta,E})+O(h^2).
\]
We may replace $\nabla\psi(G_{\eta,E})$ by
$\nabla\psi(G_{\mathrm{within}})$ with an additional $O(h^2)$
error: the two arguments differ by $O(h)$ and $\eta(E-1)\le h$.
Finally, Taylor expansion at $G_{\mathrm{within}}$ yields
\[
\mathcal H[G_{\eta,E}-G_{\mathrm{within}}]
=\frac{\eta(E-1)}2\nabla\psi(G_{\mathrm{within}})+O(h^2).
\]
Applying $\mathcal H^{-1}$ proves
\[
G_{\eta,E}-G_{\mathrm{within}}
=\eta(E-1)\Delta+O(h^2),
\qquad
\Delta=\frac12\mathcal H^{-1}
       [\nabla\psi(G_{\mathrm{within}})].
\]
The uniform remainder gives the asserted bound with a constant $K$.
For $E=1$, the exact one-step update fixes
$G_{\mathrm{within}}$. Uniqueness in the ball therefore gives
$G_{\eta,1}=G_{\mathrm{within}}$ exactly.

\medskip
\noindent\emph{5. Local convergence of the Gram iteration.}
We now establish contraction of the actual update $\Phi^G$,
rather than the auxiliary map $\mathcal K$.
Let $r_i$ be the eigenvalues of $R$ and
$\beta_i=\alpha r_i-\lambda_H>0$ those of $G_{\mathrm{within}}$.
In the corresponding symmetric matrix basis, the eigenvalues
of $\mathcal T$ are
\[
\theta_{ij}
=\frac{\lambda_W}{2\alpha}
 (\beta_i+\beta_j)(r_i^{-1}+r_j^{-1}).
\]
Put $s=\sqrt{\lambda_H\lambda_W}$ and $r_{\min}=\min_i r_i$.
Then
\[
\begin{aligned}
\theta_{ij}
&=\frac{\lambda_W}{2}
\left[
\frac{(r_i+r_j)^2}{r_ir_j}
-2s(r_i^{-1}+r_j^{-1})
\right]\\
&\ge2\lambda_W(1-s/r_{\min})=\theta_{\min}>0.
\end{aligned}
\]
The inequality uses $(r_i+r_j)^2/(r_ir_j)\ge4$ and
$r_i^{-1}+r_j^{-1}\le2/r_{\min}$.
Equality holds when $r_i=r_j=r_{\min}$, so this is the
smallest eigenvalue of $\mathcal T$.

Because $\mathcal T$ is self-adjoint, for
$h\le1/(2\max_{i,j}\theta_{ij})$,
\[
\norm{\operatorname{Id}-2h\mathcal T}_{\mathrm{op}}
=1-2h\theta_{\min}.
\]
By \eqref{sub:derivative-estimate} and smoothness,
\[
D\Phi^G(G)
=\operatorname{Id}-2h\mathcal T
 +O\!\left(h\norm{G-G_{\mathrm{within}}}_F+h^2\right).
\]

Since $G_{\eta,E}-G_{\mathrm{within}}=O(h)$, there is a
constant $C_0>0$ such that
\[
\norm{G_{\eta,E}-G_{\mathrm{within}}}_F\le C_0h.
\]
For any $G$ satisfying
$\norm{G-G_{\eta,E}}_F\le r_{\mathrm{att}}$, we therefore have
\[
\norm{G-G_{\mathrm{within}}}_F
\le r_{\mathrm{att}}+C_0h.
\]
Thus the ball centered at $G_{\eta,E}$ lies inside the
uniqueness ball whenever $r_{\mathrm{att}}+C_0h\le r$. On this ball, the derivative estimate above gives
\[
\norm{D\Phi^G(G)}_{\mathrm{op}}
\le
1-2h\theta_{\min}
+C_1h r_{\mathrm{att}}+C_2h^2
\]
for constants $C_1,C_2>0$ independent of $\eta$, $E$ and $P$.
Choose a fixed radius $r_{\mathrm{att}}>0$ small enough that
\[
r_{\mathrm{att}}\le\frac r2,
\qquad
C_1r_{\mathrm{att}}\le\frac{\theta_{\min}}2.
\]
Then decrease $\varepsilon_0$, if necessary, so that all
previous small-$h$ conditions remain valid and
\[
C_0\varepsilon_0\le\frac r2,
\qquad
C_2\varepsilon_0\le\theta_{\min}.
\]
For every $0<h\le\varepsilon_0$, these choices ensure that
the ball lies inside the uniqueness ball and that
\[
\sup_{\norm{G-G_{\eta,E}}_F\le r_{\mathrm{att}}}
\norm{D\Phi^G(G)}_{\mathrm{op}}
\le 1-\frac12h\theta_{\min}
=:q,
\qquad 0<q<1.
\]
Since the ball is convex and
$\Phi^G(G_{\eta,E})=G_{\eta,E}$, the mean-value bound gives
\[
\begin{aligned}
\norm{\Phi^G(G)-G_{\eta,E}}_F
=\norm{\Phi^G(G)-\Phi^G(G_{\eta,E})}_F \le q\norm{G-G_{\eta,E}}_F \le q r_{\mathrm{att}}<r_{\mathrm{att}}.
\end{aligned}
\]
Thus each update decreases the distance to the fixed point and keeps the Gram matrix inside the same ball. The same estimate therefore applies at every subsequent round.
For any initial Gram satisfying
$\norm{G_0-G_{\eta,E}}_F\le r_{\mathrm{att}}$, iteration yields
\[
\norm{G_t-G_{\eta,E}}_F
\le q^t\norm{G_0-G_{\eta,E}}_F
=
\left(1-\frac12h\theta_{\min}\right)^t
\norm{G_0-G_{\eta,E}}_F,
\]
which proves the claimed local convergence.
Although $G_{\eta,E}$ depends on $\eta$ and $E$,
the radius $r_{\mathrm{att}}$ and the upper bound
$\varepsilon_0$ can be chosen independently of $\eta$, $E$ and $P$.
\end{proof}

\subsection{Dependence of the displacement on client covariances}
\label{sub:covariance-dependence}
For small $h$, Theorem~\ref{sub:finite} gives $G_{\eta,E}-G_{\mathrm{within}}=\eta(E-1)\Delta+O(h^2)$. To understand how differences among client covariances affect this displacement, we here evaluate the coefficient $\Delta$ in terms of these covariances. To obtain an explicit prediction, we treat the case of commuting client covariances.

Assume that the client covariances commute pairwise.
In a common orthonormal eigenbasis, write
\[
\Sigma_m=\operatorname{diag}(\lambda_{m,1},\ldots,\lambda_{m,C}),
\qquad
\bar\lambda_i=\sum_m p_m\lambda_{m,i}.
\]
Define
\[
r_i=\sqrt{\bar\lambda_i},
\qquad
\alpha=\sqrt{\lambda_H/\lambda_W},
\qquad
s=\sqrt{\lambda_H\lambda_W}.
\]
The reference Gram is diagonal in this basis, with entries
\[
\bar g_i:=(G_{\mathrm{within}})_{ii}
=\alpha r_i-\lambda_H.
\]
To measure the covariance differences in each coordinate, put
\[
e_{m,i}=\frac{\lambda_{m,i}}{\bar\lambda_i}-1,
\qquad
\bar e_i=\sum_m p_m e_{m,i}^2.
\]
The definition of $\bar\lambda_i$ implies
$\sum_m p_m e_{m,i}=0$.

For a diagonal Gram $G=\operatorname{diag}(g_1,\ldots,g_C)$,
the $i$th diagonal entry of $\Lambda_m(G)$ is
\[
\ell_{m,i}(g_i)
=\frac{\lambda_W}{2}
-\frac{\lambda_H\lambda_{m,i}}{2(g_i+\lambda_H)^2}.
\]
Its derivative is
\[
\ell'_{m,i}(g_i)
=\frac{\lambda_H\lambda_{m,i}}{(g_i+\lambda_H)^3}.
\]
The definition of $\psi$ therefore reduces to
\[
\psi(G)
=2\sum_i g_i\sum_m p_m\ell_{m,i}(g_i)^2.
\]
At $g_i=\bar g_i$, using $\bar g_i+\lambda_H=\alpha r_i$
and $\lambda_{m,i}=r_i^2(1+e_{m,i})$, we obtain
\[
\ell_{m,i}(\bar g_i)
=-\frac{\lambda_W}{2}e_{m,i},
\qquad
\ell'_{m,i}(\bar g_i)
=\frac{\lambda_W}{\alpha r_i}(1+e_{m,i}).
\]

Differentiating $\psi$ with respect to $g_i$ now gives
\[
\begin{aligned}
\relax[\nabla\psi(G_{\mathrm{within}})]_{ii}
=2\sum_m p_m\ell_{m,i}(\bar g_i)^2
 +4\bar g_i\sum_m p_m
   \ell_{m,i}(\bar g_i)\ell'_{m,i}(\bar g_i)
   =\lambda_W^2
  \left(-\frac32+\frac{2s}{r_i}\right)\bar e_i.
\end{aligned}
\]
The relations $\sum_m p_m e_{m,i}=0$ and
$\bar g_i/(\alpha r_i)=1-s/r_i$ are used.

The full gradient at $G_{\mathrm{within}}$ is diagonal as well.
Indeed, for every diagonal matrix
$D=\operatorname{diag}(\pm1,\ldots,\pm1)$, we have
$\psi(DGD)=\psi(G)$ and
$DG_{\mathrm{within}}D=G_{\mathrm{within}}$.
Differentiating this invariance gives
\[
D\nabla\psi(G_{\mathrm{within}})D
=\nabla\psi(G_{\mathrm{within}}).
\]
Since this holds for every choice of signs, all off-diagonal
entries of the gradient vanish.

Recall that the displacement coefficient satisfies
\[
2\mathcal H[\Delta]=\nabla\psi(G_{\mathrm{within}}).
\]
In the common eigenbasis, the Hessian acts entrywise as
\[
(\mathcal H[X])_{ij}
=\frac{\lambda_W}{2\alpha}
(r_i^{-1}+r_j^{-1})X_{ij}.
\]
All these coefficients are positive, so a diagonal
$\nabla\psi(G_{\mathrm{within}})$ gives a diagonal $\Delta$.
Writing $\Delta_i:=\Delta_{ii}$, we obtain
\[
\begin{aligned}
\Delta_i
=\frac{\alpha r_i}{2\lambda_W}
  [\nabla\psi(G_{\mathrm{within}})]_{ii}=\frac{\alpha r_i\lambda_W}{2}
  \left(-\frac32+\frac{2s}{r_i}\right)\bar e_i=s\left(s-\frac34r_i\right)\bar e_i,
\end{aligned}
\]
where the last equality uses $\alpha\lambda_W=s$.
Substituting the definitions of $s$, $r_i$ and $\bar e_i$
gives
\begin{equation}
\Delta_i
=\sqrt{\lambda_H\lambda_W}
\left(
\sqrt{\lambda_H\lambda_W}
-\frac34\sqrt{\bar\lambda_i}
\right)
\sum_m p_m
\left(\frac{\lambda_{m,i}}{\bar\lambda_i}-1\right)^2.
\label{sub:drift-sign-main}
\end{equation}

At fixed $\bar\lambda_i$, this coefficient is proportional
to the weighted variance $\bar e_i$ of the relative
client covariance deviations.
If $\bar e_i=0$, then $\Delta_i=0$.
For $\bar e_i>0$, $\Delta_i$ is positive when
$\bar\lambda_i<16\lambda_H\lambda_W/9$, zero at equality,
and negative when $\bar\lambda_i>16\lambda_H\lambda_W/9$.

For sufficiently small $h$, Theorem~\ref{sub:finite} gives
\[
(G_{\eta,E}-G_{\mathrm{within}})_{ii}
=\eta(E-1)\Delta_i+O(h^2).
\]
This predicted displacement $\eta(E-1)\Delta_i$ is compared with the numerically measured displacement as shown in Figure~\ref{sub:fig-finite}(d).

\subsection{Numerical checks of finite-step dynamics}
\label{sub:num-finite}

We numerically examine the fixed points of the profiled Gram update.
Figure~\ref{sub:fig-finite}(a--c) compares them with the one-step
reference $G_{\mathrm{within}}$ and the BW barycenter $G_\star$
predicted by Theorem~\ref{sub:oracle} under exact closest-head selection.
Panels (d) and (e) test the small-$h$ predictions for the displacement
and local convergence rate, where $h=\eta E$.
Below, $G_{\eta,E}$ denotes the numerically obtained fixed point,
including in experiments beyond the small-$h$ regime.

Panels (a--c) use three random noncommuting instances with $C=3$
and $M=4$, five learning rates
$\eta\in\{0.02,0.05,0.1,0.2,0.5\}$, and ten local step counts
$E$ ranging from $1$ to $2000$.
For each pair $(\eta,E)$, we iterate the server update until
the server Gram residual is below $10^{-13}$.

\begin{figure}[!htbp]
\centering
\includegraphics[width=\linewidth]{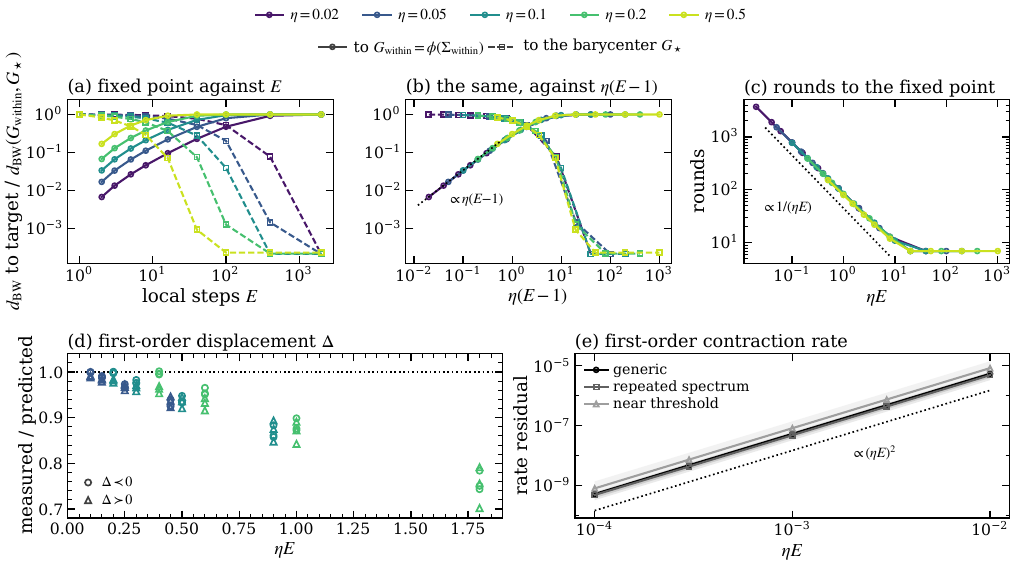}
\caption{Numerical checks of finite-step profiled dynamics.
(a) BW distances from $G_{\eta,E}$ to $G_{\mathrm{within}}$
(circles, solid) and $G_\star$ (squares, dashed), normalized by
$d_{\mathrm{BW}}(G_{\mathrm{within}},G_\star)$;
medians over three instances with $C=3$ and $M=4$.
(b) The same distances plotted against $\eta(E-1)$;
the dotted guide has slope one.
(c) Communication rounds needed to reach the stopping criterion.
(d) Measured displacements divided by the predictions
$\eta(E-1)\Delta_i$ in 24 commuting configurations, colored by $\eta$.
(e) Absolute errors in the first-order prediction
$1-2h\theta_{\min}$ for the Jacobian spectral radius;
medians and interquartile ranges within each instance family
over 780 stability records. The dotted guide has slope two.}
\label{sub:fig-finite}
\end{figure}

\paragraph{Fixed points and the BW prediction.}
Panels (a) and (b) show how the fixed point changes as local
training increases. All distances in these panels are normalized
by $d_{\mathrm{BW}}(G_{\mathrm{within}},G_\star)$.
For $E=1$, the numerical fixed point agrees with
$G_{\mathrm{within}}$ to within $3.5\times10^{-6}$ in these units.
For small $h$, its distance from $G_{\mathrm{within}}$ grows
approximately linearly with $\eta(E-1)$, consistent with
Theorem~\ref{sub:finite}. Panel (b) makes this scaling visible
by plotting the same data against $\eta(E-1)$.

With more local training, the fixed points in these examples
move away from $G_{\mathrm{within}}$ and closer to $G_\star$.
The distances to the two references cross near $\eta(E-1)=2$,
and the normalized distance to $G_\star$ eventually levels off
between $1.4\times10^{-4}$ and $8.3\times10^{-4}$.
Thus finite-step profiled training closely approaches the BW
prediction in these examples, even though it does not explicitly
impose the selection rule assumed in Theorem~\ref{sub:oracle}.

\paragraph{Communication rounds.}
Panel (c) shows the communication cost of reaching the fixed point.
For small $h$, the median round count decreases approximately
as $1/(\eta E)$, from $3698$ rounds at $\eta E=0.02$.
It levels off at $7$ rounds for $\eta E\ge20$.
The number of local head steps per client is $E$ times the round
count, so fewer communication rounds need not mean fewer local updates.

\paragraph{Local optimality and head selection.}
At $E=2000$ and $\eta\in\{0.1,0.5\}$, we check whether the
returned heads are both locally optimal and close to those prescribed
by the selection rule.
Writing $W$ for the broadcast and $U_m$ for client $m$'s return,
the maximum relative Gram error satisfies
\[
\max_m
\frac{\norm{U_mU_m^\top-G_m}_F}{\norm{G_m}_F}
<4\times10^{-15}.
\]
In contrast, the selection error
\[
\max_m
\frac{\norm{U_m-\Pi_m(W)}_F}{\norm{\Pi_m(W)}_F}
\]
ranges from $6\times10^{-5}$ to $3\times10^{-4}$ across
the three noncommuting instances.
The heads therefore have nearly optimal local Grams but differ
from the closest optimal heads.

This distinction is possible because successive gradient steps
multiply the head by symmetric matrices whose ordered product
need not be symmetric. This product can differ from the symmetric
multiplier $T_m(WW^\top)$ defining $\Pi_m(W)$ in
Lemma~\ref{sub:selection-main}.
The selection errors accompany the nonzero distance plateau above;
we also observe continued head motion after the server Gram
stabilizes in the noncommuting cases with $E>1$.
For comparison, in one commuting instance with $E=2000$ and
$\eta\in\{0.1,0.5\}$, the selection error is below $7\times10^{-15}$.
The BW distance from $G_{\eta,E}$ to $G_\star$, normalized by
$d_{\mathrm{BW}}(G_{\mathrm{within}},G_\star)$ as in panels (a)
and (b), is below $1.14\times10^{-6}$.

\paragraph{Testing the displacement formula.}
Panel (d) tests the coefficient $\Delta_i$ in
\eqref{sub:drift-sign-main} using 24 commuting configurations.
In their common eigenbasis, the first-order prediction is
\[
(G_{\eta,E}-G_{\mathrm{within}})_{ii}
\approx\eta(E-1)\Delta_i.
\]
The panel plots the measured displacement divided by this
prediction, so a ratio of one indicates agreement.
Across all 72 coordinates, including positive and negative
predicted displacements, every measured sign agrees with
the prediction. For the eight configurations with $h\le0.25$,
the ratios lie in $[0.96,1.00]$.

\paragraph{Testing the local convergence rate.}
Panel (e) tests the first-order prediction
$1-2h\theta_{\min}$ for the Jacobian spectral radius at the
fixed point, derived in Appendix~\ref{sub:finite-derivation}.
It plots the absolute error
\[
\left|
\operatorname{spr}\!\left(D\Phi_{\eta,E}^G(G_{\eta,E})\right)
-(1-2h\theta_{\min})
\right|,
\]
where $\operatorname{spr}$ denotes spectral radius.
This prediction concerns the derivative at the fixed point;
the bound $1-\tfrac12h\theta_{\min}$ in
Theorem~\ref{sub:finite} instead controls contraction
throughout a neighborhood.

We evaluate 780 stability records covering generic instances,
instances with repeated eigenvalues, and fully active instances
near the activation threshold, using $E\in\{2,4,8\}$ and
\[
h\in\{10^{-4},3\times10^{-4},10^{-3},3\times10^{-3},10^{-2}\}.
\]
All measured spectral radii are below one.
The prediction errors are consistent with $O(h^2)$ scaling,
with a median fitted log--log slope of $1.999$.
Slopes are fitted separately for each instance and $E$, using
$h\le10^{-3}$ for near-threshold instances and all five values
otherwise. Panel (e) displays all five values in every case.

\end{document}